\documentclass[11pt]{article}

\usepackage{amsmath,amssymb,amsthm}
\usepackage{mathtools}
\usepackage{bbm}
\usepackage{graphicx}
\usepackage{booktabs}
\usepackage{enumitem}
\usepackage{subcaption}
\usepackage{changepage}
\usepackage{multirow}
\usepackage{setspace}
\usepackage{tabularx}
\usepackage{algorithm}
\usepackage{algorithmic}
\usepackage{arxiv}

\newcommand{\BCINPUT}{\STATE \textbf{Input: }}
\newcommand{\BCOUTPUT}{\STATE \textbf{Output: }}

\theoremstyle{plain}
\newtheorem{theorem}{Theorem}
\newtheorem{lemma}{Lemma}

\theoremstyle{definition}

\theoremstyle{remark}

\title{Bellman Calibration for $V$-Learning in Offline Reinforcement Learning}

\author{
  Lars van der Laan \\
  Department of Statistics, University of Washington \\
  \texttt{lvdlaan@uw.edu} \\
  \And
  Nathan Kallus \\
  Netflix and Cornell University
}

\begin{document}

\maketitle
\newcommand{\HopperQRunLabel}{User-selected hopper_q_calibration_deadline_final_nonlin_hopper192_s20}
\newcommand{\HopperQUnits}{220}
\newcommand{\HopperQSeeds}{20}
\newcommand{\HopperQPolicies}{11}
\newcommand{\HopperQNeuralLinearCal}{0.475}
\newcommand{\HopperQNeuralLinearMSE}{0.409}
\newcommand{\HopperQNeuralIsoHistCal}{0.393}
\newcommand{\HopperQNeuralIsoHistMSE}{0.318}
\newcommand{\HopperQNeuralIsoHistOPE}{0.973}

\begin{abstract}
Reliable long-horizon value prediction is difficult in offline reinforcement
learning because fitted value methods combine bootstrapping, function
approximation, and distribution shift, while standard guarantees often require
Bellman completeness or realizability. We introduce \emph{Bellman calibration},
a weak reliability criterion requiring that states assigned similar predicted
values have average Bellman targets that agree with those predictions. This
criterion yields a scalar calibration error for diagnosing systematic numerical
miscalibration, which we estimate from off-policy data using doubly robust
Bellman target estimates. We then propose \emph{Iterated Bellman Calibration}, a
model-agnostic post-hoc procedure that recalibrates any learned value predictor
by fitting a one-dimensional map of its original prediction, with histogram and
isotonic variants. We prove finite-sample guarantees showing that Bellman
calibration error is controlled at one-dimensional nonparametric rates without
Bellman completeness or value-function realizability. Our value-error bounds
separate statistical estimation, finite-iteration, and approximation errors, clarifying when calibration improves value prediction and when
its gains are limited by the information in the original predictor or insufficient coverage.
\end{abstract}
\section{Introduction}

Many applications require predicting the long-term consequences of a decision
policy in a sequential, stochastic environment. We study this problem in a
Markov decision process (MDP), where the goal is to estimate the long-run value
of a target policy~\(\pi\), possibly from data collected under a different
behavior policy. This off-policy value-prediction problem arises throughout RL
and in many application domains: clinicians anticipate long-term outcomes under
proposed treatment rules \citep{van2019calibration}; online platforms estimate
customer lifetime value and retention under alternative recommendation
strategies \citep{theocharous2015ad,maystre2022temporally,xue2025auro}; and
economists study downstream effects of counterfactual programs
\citep{rust1987optimal,cowgill2019economics}. In these settings, systematic
miscalibration can distort policy comparisons, misstate long-run effects, and
lead to poor decisions.

Modern predictors such as neural networks and gradient-boosted trees are often
miscalibrated: their numerical predictions may systematically differ from the
quantities they are intended to estimate because of misspecification, limited
data, or distribution shift
\citep{zadrozny2001obtaining,niculescu2005predicting,bella2010calibration,guo2017calibration}.
For long-horizon value prediction, this concern is amplified by bootstrapping:
fitted value methods regress on targets that themselves depend on previous
value estimates, so finite-iteration error, approximation error, and off-policy
shift can propagate into the final value scale
\citep{gordon1995stable,tsitsiklis1996analysis,munos2008finite,
farahmand2010error,agarwal2021deep}. Existing guarantees typically control
these errors by imposing strong structure, such as Bellman completeness,
realizability, linear approximation, or finite state spaces
\citep{chen2019information,xie2022role,meyn2024projected,xie2021batch}. Min--max
alternatives relax some aspects of fitted value iteration but introduce their
own challenges, including saddle-point optimization, rich critic or dual
classes, and strong coverage requirements
\citep{dai2018sbeed,uehara2020minimax,jin2021bellman,xie2021batch}.

Calibration is a classical reliability property: among units assigned the same
predicted value, the corresponding outcomes should agree with that value on
average. In population terms, a calibrated predictor satisfies
\(f(X) = E[Y \mid f(X)]\). This suggests an RL analogue. Informally, a value
predictor is calibrated if, among states assigned the same predicted value, the
average target-policy Bellman backup matches that value; that is,
\(v(S) = E_\pi\{R + \gamma v(S') \mid v(S)\}\). This is a weak but meaningful
form of Bellman consistency: it constrains only the scalar prediction rather
than the full state-to-value mapping, while still ruling out systematic under-
or overestimation within level sets of the prediction.

Despite its central role in supervised learning, where predictions are
calibrated against observed labels, empirical frequencies, or conditional means
\citep{lichtenstein1977calibration,platt1999probabilistic,zadrozny2001obtaining,
vovk2005algorithmic,gupta2020distribution,gupta2021distribution,
whitehouse2024orthogonal}, calibration remains comparatively underexplored for
long-horizon value prediction in RL. Table~\ref{tab:rl-calibration-motivation}
highlights why this is problematic: value estimates can be highly informative
yet numerically misaligned because of temporal shift, off-policy fitting, finite
Bellman iteration, or approximation and optimization error. Unlike in supervised
learning, however, the calibration target is not directly observed; long-run
returns must be inferred from one-step transitions through Bellman structure.
We therefore study post-hoc Bellman calibration: a restricted recalibration of a
pretrained value predictor that corrects systematic scale errors without a new
high-dimensional Bellman fit, with finite-sample guarantees that do not require
realizability or Bellman completeness.

\begin{table}[t]
\centering
\small
\setlength{\tabcolsep}{4pt}
\renewcommand{\arraystretch}{1.08}
\caption{Why RL value predictions can be informative but poorly calibrated.}
\label{tab:rl-calibration-motivation}
\begin{tabularx}{\linewidth}{p{0.47\linewidth} p{0.45\linewidth}}
\toprule
\textbf{Failure mode} & \textbf{Calibration correction} \\
\midrule

Temporal or distribution shift: value estimates reflect past rewards, policies,
or environments rather than the current target setting
\citep{kumar2020conservative}.
&
Update the value scale using recent or gold-standard trajectories, while
retaining signal from the historical model.
\\
\midrule

Off-policy fitting: a high-capacity value model trained on large
behavior-policy data can be informative but mis-scaled for the target-policy
state--action distribution \citep{fujimoto2019off}.
&
Use the large off-policy sample to learn the value predictor, then recalibrate its
scale on a smaller on-policy or target-policy reference sample.
\\
\midrule

Finite Bellman iteration: computational limits can force early stopping or a
shorter effective horizon, mis-scaling values for the intended return target
\citep{munos2008finite}.
&
Re-align predictions to the intended horizon without rerunning the full Bellman
iteration.
\\
\midrule

Approximation and optimization error: misspecified value classes, restricted
critics, or unstable saddle-point training can yield biased value predictions
\citep{uehara2021finite}.
&
Remap the biased but informative prediction to observed or Bellman-consistent values,
rather than requiring exact specification.
\\

\bottomrule
\end{tabularx}
\end{table}

\textbf{Contributions.}
We connect two largely separate literatures: post-hoc calibration for
supervised prediction and long-horizon value learning in RL. Existing work gives
rich methods for learning value functions, but much less theory for
model-agnostic reliability assessment and post-hoc correction of already learned
value predictors. Our contributions are:

\begin{enumerate}[leftmargin=1.5em,topsep=0.2em,itemsep=0.25em,parsep=0pt,partopsep=0pt]
\item We formalize \textbf{Bellman calibration}, a dynamic analogue of
supervised calibration for infinite-horizon value prediction. Bellman consistency
implies calibration, but calibration is weaker: predictions need only match
average Bellman backups among states assigned similar values.
\item We propose \textbf{\(\ell^2\)-Bellman calibration error} as a scalar
diagnostic for systematic value miscalibration and develop debiased off-policy
estimators of this error. We also prove a \textbf{calibration-refinement} bound
showing that value error decomposes into calibration error and the approximation
error of the best predictor based only on the original model output.
\item We introduce \textbf{post-hoc Bellman calibration} via histogram binning
and isotonic regression. The methods apply to any value predictor and use
doubly robust Bellman targets for off-policy data.
\item We prove \textbf{finite-sample guarantees} for histogram calibration,
controlling both \(\ell^2\)-Bellman calibration error and value error. The
bounds require neither Bellman completeness nor realizability and separate
statistical, auxiliary-model, and iteration errors.
\end{enumerate}

\subsection{Related work}

\textbf{Post-hoc calibration.}
Our work builds on post-hoc calibration for supervised prediction
\citep{zadrozny2001obtaining,zadrozny2002transforming,niculescu2005predicting,
bella2010calibration,guo2017calibration}; see \citet{roth2022uncertain} for an
overview. This literature calibrates static predictions against observed labels,
empirical frequencies, or conditional means. We study long-horizon value
predictions in MDPs, where returns are not directly observed and must be inferred
through Bellman structure. We are also related to causal calibration for static
treatment-effect prediction \citep{van2023causal,whitehouse2024orthogonal}, but
our setting is dynamic and long-horizon, with calibration targets defined by
value functions rather than one-step conditional effects. The closest work is \citep{van2025automaticDRL}, who use isotonic post-hoc
calibration of \(Q\)-functions inside Double RL policy-value estimators
\citep{kallus2020double,kallus2022efficiently}. Their calibration step is used
as an auxiliary-model correction for a global value estimator. We instead study calibrated value prediction directly and give finite-sample
guarantees for both Bellman calibration error and value error. The same framework also applies to \(Q\)-function calibration by replacing
importance-weighted FVI with unweighted FQE and treating each state-action pair
as the state.

\textbf{Relation to Bellman learning and fine-tuning.}
Bellman calibration complements methods that change value learning itself,
including minimax or adversarial Bellman objectives
\citep{uehara2020minimax,uehara2023offline,xie2021batch,di2023pessimistic},
conservative updates \citep{kumar2020conservative,an2021uncertainty}, and
structural assumptions such as linearity, low rank, representations, or
discretization
\citep{shah2020sample,chang2022learning,peng1993convergence,van2006performance}.
It can be applied after any such method because it only needs scalar value
predictions and a reference sample. Unlike offline-to-online transfer,
critic/policy fine-tuning, or safe policy improvement
\citep{rusu2016progressive,nair2020awac,lee2022offline,laroche2019safe},
Bellman calibration fixes the target policy, value model, representation, and
Bellman learner, and fits only a one-dimensional post-hoc map to the target
value scale. Thus it is an output-scale correction rather than RL fine-tuning,
with finite-sample guarantees for calibration and value error. Related work on
calibrated dynamics uncertainty, conformal sets, and distributional RL targets
concerns different objects
\citep{malik2019calibrated,zhang2023conformal,sun2023conformal,
bellemare2017distributional,dabney2018implicit}.

\section{Calibration for Value Functions}
\subsection{Setup and Notation}

We consider an MDP with state space \(\mathcal S\), discrete action space
\(\mathcal A\), initial-state distribution \(\rho\) for \(S_0\), transition kernel
\(P(s'\mid s,a)\), reward function \(r_0(s,a)\), and discount factor
\(\gamma\in[0,1)\). Data are collected under a behavior policy
\(b_0(a\mid s)\). Each calibration sample is an i.i.d.\ transition
\((S,A,R,S')\), where \((S,A)\sim \rho\times b_0\),
\(S'\sim P(\cdot\mid S,A)\), and
\(R=r_0(S,A)+\varepsilon\) with \(E[\varepsilon\mid S,A]=0\). We observe a calibration sample
\(\mathcal C_n:=\{(S_i,A_i,R_i,S_i')\}_{i=1}^n\).

Let \(\pi(a\mid s)\) be the target policy. The value function under \(\pi\) is \(v_0(s)
=
E_{\pi}[\sum_{t=0}^\infty \gamma^t R_t
\,|\, S_0=s]\), where \(E_\pi\) denotes the law induced by
\(A_t\sim\pi(\cdot\mid S_t)\) and
\(S_{t+1}\sim P(\cdot\mid S_t,A_t)\).

For a state-action function \(f\) and state function \(v\), define
\[
\begin{aligned}
(\pi f)(s)&:=\sum_{a\in\mathcal A}\pi(a\mid s)f(s,a),\\
(b_0 f)(s)&:=\sum_{a\in\mathcal A}b_0(a\mid s)f(s,a),\\
(Pv)(s,a)&:=E[v(S')\mid S=s,A=a].
\end{aligned}
\]
Write
\[
w_\pi:=\pi/b_0,\qquad
r_{0,\pi}:=\pi r_0,\qquad
P_\pi:=\pi P,\qquad
\mathcal T_\pi v := r_{0,\pi}+\gamma P_\pi v
\]
for the policy Bellman operator. The value function is the unique bounded fixed
point \citep{bellman1952theory,bellman1966dynamic}
\begin{equation}
\label{eqn::bellman}
v_0=\mathcal T_\pi v_0 .
\end{equation}
For a measure \(\mu\) on \(\mathcal S\), write
\[
\|f\|_{2,\mu}
:=
\left\{\int f(s)^2\,\mu(ds)\right\}^{1/2}.
\]
Let \(L^2(\rho)\) denote the space of square-integrable state functions under
\(\rho\). Let \(\|f\|_{n,S,A}\) and \(\|g\|_{n,S'}\) denote empirical \(L^2\)
norms over \(\{(S_i,A_i)\}_{i=1}^n\) and \(\{S_i'\}_{i=1}^n\), respectively.
We write \(\lesssim\) for inequalities up to absolute constants and
\([B]:=\{1,\ldots,B\}\).

\subsection{Bellman Calibration: A Minimal Requirement}

Let \(\hat v\) be an estimated value function for the target policy \(\pi\),
fit on data independent of the calibration sample \(\mathcal C_n\). Our goal is
not to replace the underlying value-learning procedure, but to impose a weak
post-hoc Bellman consistency requirement on its output. Throughout, we call the
scalar \(\hat v(S)\) the predicted value; in the supervised calibration
literature, this same scalar is often called a prediction score.

A natural starting point is \emph{statewise Bellman consistency}:
\[
\hat v(S)
\approx
E_\pi\!\left[R+\gamma \hat v(S')\mid S\right].
\]
Equivalently, the Bellman residual
\(R+\gamma \hat v(S')-\hat v(S)\) should have conditional mean zero given the
full state. In high-dimensional or continuous state spaces, this is a strong
requirement: checking or enforcing it requires estimating a state-dependent
conditional mean and is closely tied to structural assumptions such as Bellman
completeness \citep{chen2019information}.

We instead require Bellman consistency only after grouping states by their
predicted value:
\[
\hat v(S)
\approx
E_\pi\!\left[R+\gamma \hat v(S')\mid \hat v(S)\right].
\]
Thus, predictions should agree on average with their Bellman targets within
groups of states assigned the same predicted value \(\hat v(S)\). This is the
dynamic analogue of post-hoc calibration in supervised prediction, where
predictions are required to be correct on average conditional on their own
numerical output
\citep{zadrozny2002transforming,niculescu2005predicting,guo2017calibration}.
The benefit is that a high-dimensional Bellman moment is reduced to a
one-dimensional conditional moment, enabling post-hoc correction without
requiring the full fitted Bellman update to be correctly specified.

\textbf{Bellman calibration.}
For any candidate value function \(v\), define the Bellman-calibration map
\begin{equation}
\label{eqn::calibrationmap}
\Gamma_0(v)(s)
:=
E_\pi\!\left[R+\gamma v(S') \mid v(S)=v(s)\right].
\end{equation}
We say that \(v\) is \emph{Bellman calibrated} if
\(v(S)=\Gamma_0(v)(S)\) almost surely, or equivalently,
\[
E_\pi\!\left[R+\gamma v(S')-v(S)\mid v(S)\right]=0
\quad\text{almost surely}.
\]

\textbf{A necessary but weak condition.}
The true value function \(v_0\) is Bellman calibrated. The Bellman equation gives
\(v_0(S)=E_\pi[R+\gamma v_0(S')\mid S]\); conditioning on the coarser predicted value
\(v_0(S)\) yields
\[
v_0(S)
=
E_\pi[R+\gamma v_0(S')\mid v_0(S)]
=
\Gamma_0(v_0)(S).
\]
Thus Bellman consistency implies Bellman calibration. The converse need not
hold: because calibration only conditions on the scalar predicted value \(v(S)\),
state-level errors can average out within predicted-value groups
\citep{gupta2020distribution,van2023causal}. Hence Bellman calibration is a
minimal reliability requirement, not a guarantee of value accuracy.

\textbf{Reward-model interpretation.}
Define the reward implied by a value predictor \(v\) as
\[
r_{\pi,v}(s):=v(s)-\gamma P_\pi v(s).
\]
This is the one-step reward under which \(v\) would satisfy the Bellman
equation for policy \(\pi\). Bellman calibration is equivalent to
\[
E_\pi\!\left[R-r_{\pi,v}(S)\mid v(S)\right]=0.
\]
Thus, among states assigned the same predicted long-run value, the reward
implied by \(v\) matches the true one-step reward on average. When
\(\gamma=0\), this reduces to ordinary regression calibration
\citep{roth2022uncertain}.

\subsection{A Calibration--Refinement Bound for Value Error}

We next relate Bellman calibration to value prediction error. The key point is
that calibration is a post-processing operation: it can correct systematic
Bellman residuals among states with the same predicted value, but it cannot
recover information lost by reducing the state \(S\) to \(\hat v(S)\).

For a candidate value function \(v\), define the \(\ell^2\) \textbf{Bellman calibration
error}
\begin{equation}
\label{eqn::calerror}
\mathrm{Cal}_{\ell^2}^2(v)
:=
\|v-\Gamma_0(v)\|_{2,\rho}^2
=
E_\rho\!\left[
  \left\{
    E_\pi[R+\gamma v(S')-v(S)\mid v(S)]
  \right\}^2
\right].
\end{equation}
This criterion is zero exactly when \(v\) is Bellman calibrated, and measures
the remaining Bellman residuals conditional on the predicted value.

The value predictors generated by the original prediction \(\hat v(S)\) form the
class
\[
\mathcal V_{\hat v}
:=
\{g\circ\hat v:g:\mathbb R\to\mathbb R \text{ measurable}\},
\qquad
\Pi_{\hat v}
:=
\text{\(L^2(\rho)\)-projection onto }\mathcal V_{\hat v}.
\]
Define the projected transition by
\[
P_{\pi,\hat v}
:=
\Pi_{\hat v}P_\pi,
\qquad
P_{\pi,\hat v}f(s)
=
E_\pi[f(S')\mid \hat v(S)=\hat v(s)].
\]

Define the discounted projected concentrability by
\[
\mathcal C_{\hat v,\gamma}
:=
\left\{(1-\gamma)\sum_{t=0}^\infty
\gamma^t
\left\|
\frac{d(\rho P_{\pi,\hat v}^t)}{d\rho}
\right\|_\infty^{1/2}
\right\}^2 .
\]

\begin{enumerate}[label=\textbf{A\arabic*}, ref={A\arabic*}, leftmargin=1.5em, series=cond]
\item \label{cond::A1}
\textbf{Projected concentrability.}
\(\mathcal C_{\hat v,\gamma}<\infty\).
\end{enumerate}
When \(\rho\) is stationary under \(P_{\pi}\), it is also stationary under
\(P_{\pi,\hat v}\) by Lemma~\ref{lemma:aggregated-stationary}, so
\(\mathcal C_{\hat v,\gamma}=1\). More generally,
\(\mathcal C_{\hat v,\gamma}\) measures distribution drift from \(\rho\): it is
small when projected rollouts remain comparable to \(\rho\), and large when they
concentrate where \(\rho\) has little mass.

\begin{theorem}[Calibration--refinement bound]
\label{theorem::calrefine}
Under \ref{cond::A1},
\[
\|\hat v-v_0\|_{2,\rho}
\le
\frac{\sqrt{\mathcal C_{\hat v,\gamma}}}{1-\gamma}
\left\{
\mathrm{Cal}_{\ell^2}(\hat v)
+
\min_{\theta:\mathbb R\to\mathbb R}
\|\theta\circ\hat v-v_0\|_{2,\rho}
\right\}.
\]
\end{theorem}

Theorem~\ref{theorem::calrefine} gives a post-processing guarantee: value error
is controlled by Bellman calibration error \(\mathrm{Cal}_{\ell^2}(\hat v)\)
plus the refinement error $\min_{\theta:\mathbb R\to\mathbb R}
\|\theta\circ\hat v-v_0\|_{2,\rho},$ the error of the best value predictor that depends on the state only through
\(\hat v(S)\). The calibration error measures systematic Bellman residuals
within levels of \(\hat v(S)\), whereas the refinement error measures the
remaining value variation not captured by \(\hat v(S)\). Thus, calibration can
remove Bellman miscalibration using only a post-processing map of the original
prediction, but it cannot recover information about \(v_0(S)\) absent from
\(\hat v(S)\). In particular, when calibration error is driven to zero, value
error is controlled, up to the projected concentrability and discount factors,
by the refinement error. This mirrors classical calibration--refinement
decompositions in classification and regression
\citep{murphy1973new,degroot1983comparison,van2023causal,whitehouse2024orthogonal}.
The factor \(\sqrt{\mathcal C_{\hat v,\gamma}}/(1-\gamma)\) plays the same role
as the concentrability and discount penalties in classical FVI/FQE error bounds
\citep{munos2008finite}.

\subsection{Evaluating Bellman Calibration Error with Doubly Robust Bellman Targets}
\label{sec:bellman-calibration-error}

We now focus on estimating \(\mathrm{Cal}_{\ell^2}(\hat v)\), the component of
Theorem~\ref{theorem::calrefine} that can be evaluated from transition data
and improved without refitting the value predictor. This makes Bellman calibration error both
a diagnostic for systematic Bellman residuals within the level sets of a fixed
predictor \(\hat v\) and an objective for post-hoc calibration.

Because data are generated under the behavior policy \(b_0\), we estimate the
statewise Bellman target using a doubly robust Bellman pseudo-outcome, extending
the calibration target estimates of \citep{xu2022calibration,van2023causal} to
dynamic settings. Let \(\widehat w_\pi\) estimate \(w_\pi:=\pi/b_0\). For a
fixed \(v\), define
\[
q_v(s,a)
:=
E\{R+\gamma v(S')\mid S=s,A=a\}.
\]
Let \(\hat q_v\) estimate \(q_v\). For \(O=(S,A,R,S')\), set
\[
\widehat Y_v(O)
:=
(\pi \hat q_v)(S)
+
\widehat w_\pi(A\mid S)
\{R+\gamma v(S')-\hat q_v(S,A)\},
\qquad
\widehat{\mathcal T}_\pi v(s)
:=
E\{\widehat Y_v(O)\mid S=s\}.
\]
The next result gives the conditional bias of this pseudo-outcome.

\begin{theorem}[Doubly robust Bellman target]
\label{theorem::DRpseudo}
For any fixed \(v\),
\[
E\!\left[
\{\widehat{\mathcal T}_\pi v-\mathcal T_\pi v\}(S)
\mid v(S)=t
\right]
=
E\!\left[
  b_0\!\left\{(w_\pi-\widehat w_\pi)(\hat q_v-q_v)\right\}(S)
  \mid v(S)=t
\right].
\]
Consequently, \(E\{\widehat Y_v(O)\mid v(S)\}
=
E\{(\mathcal T_\pi v)(S)\mid v(S)\}\) if either
\(\widehat w_\pi=w_\pi\) or \(\hat q_v=q_v\).
\end{theorem}

This construction includes importance weighting \((\hat q_v=0)\) and
model-based Bellman-target evaluation \((\widehat w_\pi=0)\) as special cases.
The theorem shows that either nuisance estimator suffices to recover the
prediction-conditional Bellman target.

To estimate the proposed calibration error, we use a \textbf{cross-fitted debiased moment estimator}.
Let \(U_i:=v(S_i)\) and \(\widehat Y_i:=\widehat Y_v(O_i)\). Following
\citet{xu2022calibration}, we estimate the squared calibration error by
\[
\widehat{\mathrm{Cal}}_{\ell^2,\mathrm{db}}^2(v)
:=
\frac1n\sum_{i=1}^n
\{\widehat Y_i-U_i\}
\{\widehat m_{-i}(U_i)-U_i\},
\]
where \(\widehat m_{-i}\) is fit without observation \(i\) to estimate
\(u\mapsto E(\widehat Y_v\mid v(S)=u)\). The population identity behind this
estimator is
\[
E\!\left[
\{\widehat Y_v-v(S)\}
\left\{E(\widehat Y_v\mid v(S))-v(S)\right\}
\right]
=
E\!\left[
\left\{E(\widehat Y_v\mid v(S))-v(S)\right\}^2
\right].
\]
By Theorem \ref{theorem::DRpseudo}, this equals
\(\mathrm{Cal}_{\ell^2}^2(v)\), up to function approximation error.

\section{Post-hoc Bellman Calibration}
\label{sec:posthoc-calibration}

\subsection{Iterated Bellman Calibration}
\label{sec:iterated-bellman-calibration}

Algorithm~\ref{alg::cal-gen} gives a post-hoc FVI-style procedure that
transforms an initial value predictor \(\hat v\) into a calibrated predictor
\(V_{\mathrm{cal}}=V_K=g_K\circ\hat v\). Starting from \(V_0=\hat v\), each
iteration forms doubly robust Bellman targets \(\widehat Y_{V_k}(O_i)\) and
regresses them on the fixed one-dimensional inputs \(\hat v(S_i)\) over a
calibrator class \(\mathcal G\), yielding \(V_{k+1}=g_{k+1}\circ\hat v\).
Equivalently, for
\(\mathcal V_{\hat v}(\mathcal G):=\{g\circ\hat v:g\in\mathcal G\}\), the
algorithm is a projected Bellman update restricted to functions of the original
prediction. This avoids high-dimensional function approximation while correcting
Bellman miscalibration in \(\hat v\); the contraction results below imply that
\(O\{\log(1/\varepsilon)/(1-\gamma)\}\) iterations suffice to reach optimization
error \(\varepsilon\).

\begin{algorithm}[h]
\caption{\textbf{Iterated Bellman Calibration}}
\label{alg::cal-gen}
\begin{algorithmic}[1]
\BCINPUT \(\hat v\), \(\mathcal C_n=\{O_i=(S_i,A_i,R_i,S_i')\}_{i=1}^n\), pseudo-outcome \(v\mapsto\widehat Y_v\), class \(\mathcal G\), iterations \(K\)
\STATE Initialize \(V_0=\hat v\)
\FOR{\(k=0,\ldots,K-1\)}
  \STATE \(V_{k+1}=g_{k+1}\circ\hat v\), where $ g_{k+1}\in\arg\min_{g\in\mathcal G}\sum_{i=1}^n
  \{\widehat Y_{V_k}(O_i)-g(\hat v(S_i))\}^2 .$
\ENDFOR
\BCOUTPUT \(V_K\)
\end{algorithmic}
\end{algorithm}

At the population level, the iteration approximates the projected Bellman
operator \(\mathcal B_{\hat v} f
:=
\Pi_{\hat v} r_{0,\pi} + \gamma P_{\pi,\hat v} f\), whose fixed point over functions of \(\hat v(S)\) has zero Bellman calibration
error. Exact finite-sample recovery of this projected fixed point would
generally require realizability in \(\mathcal V_{\hat v}(\mathcal G)\) and
Bellman completeness of the fitted update \citep{munos2008finite}. We instead
seek post-hoc Bellman calibration guarantees without imposing these conditions.
We therefore instantiate Algorithm~\ref{alg::cal-gen} with one-dimensional
histogram and isotonic calibrators \citep{stone1977consistent,barlow1972isotonic}.
These are dynamic analogues of histogram binning
\citep{zadrozny2001obtaining,zadrozny2002transforming,gupta2021distribution}
and isotonic calibration
\citep{zadrozny2002transforming,niculescu2005predicting,van2023causal,
van2025generalized}, and they admit finite-sample Bellman calibration
guarantees without Bellman completeness.

\subsection{Histogram Calibration}
\label{sec:hist}

Histogram calibration instantiates Algorithm~\ref{alg::cal-gen} with a
piecewise-constant calibrator. Partition the initial values
\(\{\hat v(S_i)\}_{i=1}^n\) into ordered bins \(I_1,\ldots,I_B\), using either
empirical quantiles (\emph{equal-mass}) or uniform discretization
(\emph{equal-width}). We take
\[
\mathcal G_B
:=
\mathrm{span}\{\mathbf 1_{I_b}:b\in[B]\},
\qquad
\mathcal H_B(\hat v)
:=
\{g\circ\hat v:g\in\mathcal G_B\},
\]
so every \(h\in\mathcal H_B(\hat v)\) is a value predictor that is constant
within bins of the original prediction. The bins may be data-adaptive, but
\(B\) is deterministic and may grow with \(n\).

With this class, each iteration simply averages the fitted Bellman targets
within each bin:
\[
g_{k+1}(t)
=
\frac{1}{n_b}
\sum_{i:\hat v(S_i)\in I_b}
\widehat Y_i^{(k)},
\qquad
t\in I_b,
\]
where \(n_b:=\sum_{i=1}^n\mathbf 1\{\hat v(S_i)\in I_b\}\). Thus, histogram
calibration is fitted Bellman iteration on a finite aggregation of the initial
predicted value. At convergence, the output approximately satisfies the empirical aggregated
Bellman equation within each bin:
\[
V_K(s)
\approx
E_n\!\left[
\widehat Y_{V_K}(O)
\,\middle|\,
\hat v(S)\in I_b
\right],
\qquad \hat v(s)\in I_b .
\]
This is the empirical analogue of Bellman calibration and is the key property
used in the next result.

 \textbf{Calibration error.}
We next give finite-sample, finite-iteration bounds for the Bellman calibration
error \eqref{eqn::calerror}. The analysis uses the following regularity
conditions.

\begin{enumerate}[label=\textbf{C\arabic*}, ref={C\arabic*}, leftmargin=1.5em, series=cond]

\item \label{cond::C1}
\textbf{Boundedness.}
\(R\), \(w_\pi\), \(\widehat w_\pi\), \(\hat v\), the Bellman targets \(q_f\),
and their estimates \(\hat q_f\) are uniformly bounded by \(M<\infty\) for all
\(f\) considered.

\item \label{cond::C1b}
\textbf{Sample splitting.}
The auxiliary estimators \(\widehat w_\pi\) and \(v\mapsto\hat q_v\), and the
initial predictor \(\hat v\) are trained on data independent of
\(\mathcal C_n\).

\item \label{cond::C2}
\textbf{Stability of the Bellman-target regression.}
For all \(f,g\in\mathcal H_B(\hat v)\) with
\(\|f\|_\infty,\|g\|_\infty\le M\),
\[
\|\hat q_f-\hat q_g\|_{n,S,A}
\le
L\,\|f-g\|_{n,S'} .
\]
\end{enumerate}

These conditions are mainly technical. Boundedness can be enforced by clipping
or truncation, and sample splitting can be replaced by cross-fitting. Condition
\ref{cond::C2} controls the complexity introduced by the map
\(v\mapsto\hat q_v\); it holds trivially for the importance-weighted target
\((\hat q_v=0)\), and holds with \(L=1\) for cellwise estimates of
\(E\{R+\gamma f(S')\mid S,A\}\).

\begin{theorem}[Calibration error for histogram binning]
\label{theorem::Calerrorhist}
Assume \ref{cond::C1}--\ref{cond::C2}. Then there is a constant
\(C<\infty\) such that, for any \(K\in\mathbb N\), with probability at least
\(1-\delta\),
\begin{align*}
\mathrm{Cal}_{\ell^2}(V_K)
\le\;&
C\left\{
\sqrt{\frac{B}{n}\log\!\left(\frac{n}{B}\right)}
+
\sqrt{\frac{\log(1/\delta)}{n}}
\right\} \\
&+
\big\|b_0\{(\widehat w_\pi-w_\pi)
(\widehat q_{V_K}-q_{V_K})\}\big\|_{2,\rho}  +
\big\|
\widehat{\mathcal T}_\pi V_K
-
\widehat{\mathcal T}_\pi V_{K-1}
\big\|_{2,\rho}.
\end{align*}
\end{theorem}

The bound separates three errors: the statistical error of one-dimensional
histogram calibration, matching classical rates when the calibration target is
observed directly
\citep{gupta2020distribution,gupta2021distribution,van2025generalized}; the
doubly robust auxiliary-model error from Theorem~\ref{theorem::DRpseudo}; and
the finite-iteration error, which decays geometrically as \(\gamma^K\) under the
contraction conditions in Appendix~\ref{appendix::convergence}, as in fitted
value iteration \citep{munos2008finite}. The leading calibration rate does not
accumulate over the horizon: \(\gamma\) enters only through the finite-iteration
and auxiliary-model terms. This reflects that Bellman calibration is a one-step
conditional prediction problem, even though translating calibration into value
accuracy still depends on discounting and coverage
(Theorem~\ref{theorem::calrefine}).

\textbf{Value error.}
We next show that post-hoc calibration does not add error beyond the
information loss from restricting to the histogram calibrator class, while it can
remove Bellman miscalibration in the original predictor.

Let \(\Pi_B\) denote the \(L^2(\rho)\)-projection onto
\(\mathcal H_B(\hat v)\), i.e.,
\[
\Pi_B q:=\arg\min_{h\in\mathcal H_B(\hat v)}\|q-h\|_{2,\rho},
\qquad
\mathcal B_{\pi,B}:=\Pi_B\mathcal T_\pi,
\qquad
V_B^\star=\mathcal B_{\pi,B}V_B^\star .
\]
Write \(P_{\pi,B}:=\Pi_B P_\pi\), so that, for
\(\hat b(s)=b\) when \(\hat v(s)\in I_b\),
\[
P_{\pi,B}f(s)
=
E_\pi[f(S')\mid \hat v(S)\in I_{\hat b(s)}].
\]
Finally, define the aggregated discounted concentrability by
\[
\mathcal C_{B,\gamma}
:=
\left\{(1-\gamma)\sum_{t=0}^\infty
\gamma^t
\left\|
\frac{d(\rho P_{\pi,B}^t)}{d\rho}
\right\|_\infty^{1/2}
\right\}^2 .
\]

\begin{enumerate}[label=\textbf{C\arabic*}, ref={C\arabic*}, leftmargin=1.5em, resume=cond]
\item \label{cond::Ccoverage}
\textbf{Aggregated projected concentrability.}
\(\mathcal C_{B,\gamma} < \infty\).
\end{enumerate}

\begin{theorem}[Value error for histogram calibration]
\label{theorem::regretHist}
Assume \ref{cond::C1}--\ref{cond::Ccoverage}. Then there exists a constant
\(C<\infty\) such that, for any \(K\in\mathbb N\), with probability at least
\(1-\delta\),
{\small
\begin{align*}
\|V_K-v_0\|_{2,\rho}
\le\;&
\gamma^K\|\hat v-V_B^\star\|_{\infty}
+
\frac{\sqrt{\mathcal C_{B,\gamma}}}{1-\gamma}
\min_{h\in\mathcal H_B(\hat v)}\|h-v_0\|_{2,\rho}
\\
&+
C\frac{\sqrt{\mathcal C_{B,\gamma}}}{1-\gamma}
\Bigg[
  \sqrt{\frac{B}{n}\log\!\left(\frac{n}{B}\right)}
  +
  \sqrt{\frac{\log(K/\delta)}{n}}
 +
\max_{0\le j\le K}
\big\|
b_0\{
(\widehat w_\pi-w_\pi)
(\widehat q_{V_j}-q_{V_j})
\}
\big\|_{2,\rho}
\Bigg].
\end{align*}
}
\end{theorem}
\textbf{Proof idea.}
The result avoids realizability and Bellman completeness for the original value
class because completeness holds after aggregation. For any
\(f\in\mathcal H_B(\hat v)\), the aggregated transition \(P_{\pi,B}f\) is
binwise constant in \(\hat v\), so \(P_{\pi,B}f\in\mathcal H_B(\hat v)\)
and hence \(\mathcal B_{\pi,B}f\in\mathcal H_B(\hat v)\). The proof uses
sup-norm contraction to obtain the aggregated fixed point, then propagates
\(\rho\)-norm one-step regression errors through a discounted resolvent
controlled by \(\mathcal C_{B,\gamma}\). It also bounds
\(V_B^\star-v_0\) by a histogram analogue of
Theorem~\ref{theorem::calrefine}. Unlike prior discretization arguments for
Bellman consistency
\citep{whitt1978approximations,peng1993convergence,van2006performance,
chen2019information,xie2021batch}, we discretize \(\hat v(S)\) to
target Bellman calibration.

\textbf{Discussion.}
Theorems~\ref{theorem::Calerrorhist} and~\ref{theorem::regretHist} separate
calibration from value accuracy. With known importance weights and negligible
iteration error, histogram calibration drives Bellman calibration error to zero
at the one-dimensional rate \(\sqrt{(B/n)\log(n/B)}\). Value error contains the
same stochastic term, but also includes the prediction-only approximation error $(\sqrt{\mathcal C_{B,\gamma}})/(1-\gamma)
\min_{h\in\mathcal H_B(\hat v)}\|h-v_0\|_{2,\rho},$ the cost of restricting the calibrated predictor to binned transformations of
the original prediction. Thus \(B\) controls the usual bias--variance tradeoff:
coarse bins may discard information in \(\hat v\), while fine bins increase
variance. The coverage factor \(\mathcal C_{B,\gamma}\) is the histogram
analogue of FVI concentrability.

It suffices to choose \(B\) by predictive cross-validation. The prediction-only approximation term
satisfies
\[
\min_{h\in\mathcal H_B(\hat v)}\|h-v_0\|_{2,\rho}
\le
\|\hat v-v_0\|_{2,\rho}
+
\min_{h\in\mathcal H_B(\hat v)}\|h-\hat v\|_{2,\rho}.
\]
Thus binning is no worse than the original predictor, up to identity
approximation on the range of \(\hat v(S)\). For equal-mass bins, the histogram
bias--variance balance suggests \(B\asymp n^{1/3}\), yielding
\citep{gyorfi2002distribution}
\[
\|V_K-v_0\|_{2,\rho}
\le
\frac{\sqrt{\mathcal C_{B,\gamma}}}{1-\gamma}
\|\hat v-v_0\|_{2,\rho}
+
O_p\{(\log n/n)^{1/3}\},
\]
with Bellman calibration error vanishing at the same rate.

\subsection{Isotonic Calibration}

Histogram calibration requires choosing the number of bins \(B\). As a
low-tuning alternative, we consider \emph{isotonic iterated Bellman
calibration}, obtained by taking \(\mathcal G=\mathcal G_{\mathrm{iso}}\) in
Algorithm~\ref{alg::cal-gen}, where \(\mathcal G_{\mathrm{iso}}\) is the class
of nondecreasing functions on \(\mathbb R\). Each update isotonic-regresses the
fitted Bellman targets on the initial predictions \(\hat v(S_i)\), producing a
monotone calibrator computable in near-linear time by the pool-adjacent-violators
algorithm (PAVA) \citep{best1990active}. Like histogram calibration, isotonic
calibration averages Bellman targets over groups of nearby predicted values,
but the groups are chosen adaptively by the monotonicity constraint
\citep{barlow1972isotonic}. This regularizes the calibration map without
selecting \(B\); since the identity map is monotone, isotonic calibration can
preserve an already reliable predictor while correcting systematic monotone
distortions.

\begin{enumerate}[label=\textbf{C\arabic*}, ref={C\arabic*}, leftmargin=1.5em, resume=cond]
\item \label{cond::C3}
\textbf{Finite variation.}
There exists \(V<\infty\) such that, almost surely, the function $t
\mapsto
E\!\left[
(\mathcal T_\pi V_K)(S)
\mid
\hat v(S)=t,\mathcal C_n
\right]$
has total variation at most \(V\).
\end{enumerate}

Condition~\ref{cond::C3} is standard in isotonic calibration
\citep{van2023causal} and controls the complexity of the one-dimensional
calibration curve. For histogram binning, the analogous complexity is governed
by the number of bins \(B\).

\begin{theorem}[Calibration error for isotonic calibration]
\label{theorem::Calerroriso}
Assume \ref{cond::C1}, \ref{cond::C1b}, \ref{cond::C2}, and \ref{cond::C3}
with \(\mathcal G=\mathcal G_{\mathrm{iso}}\). Then there exists a constant
\(C<\infty\) such that, for any \(K\in\mathbb N\), with probability at least
\(1-\delta\),
{\small
\begin{align*}
\mathrm{Cal}_{\ell^2}(V_K)
\le\;&
C\left\{
n^{-1/3}
+
\sqrt{\frac{\log(1/\delta)}{n}}
\right\}  +
\big\|
b_0\{
(\widehat w_\pi-w_\pi)
(\widehat q_{V_K}-q_{V_K})
\}
\big\|_{2,\rho}  +
\big\|
\widehat{\mathcal T}_\pi V_K
-
\widehat{\mathcal T}_\pi V_{K-1}
\big\|_{2,\rho}.
\end{align*}
}
\end{theorem}

The isotonic analysis parallels the histogram case but uses the adaptive pooling
induced by PAVA. Monotonicity is only a pooling regularizer; we do not assume the
calibration target is monotone. The resulting \(n^{-1/3}\) rate matches classical
isotonic regression \citep{chatterjee2013improved} and histogram binning with
\(B\asymp n^{1/3}\), without choosing \(B\). However, fully iterated isotonic
calibration does not inherit our histogram value-error bound without Bellman
completeness, because the PAVA partition can change across iterations.

Appendix~\ref{appendix::hybrid} therefore introduces an isotonic-histogram
hybrid: learn the isotonic partition once, then run the histogram update on that
fixed partition. This preserves adaptive pooling while restoring the
fixed-partition structure required by our finite-sample value-error analysis,
and is the procedure we recommend in practice.

\section{Experimental investigation}
\label{sec:experiments}

We evaluate Bellman calibration as a post-hoc correction for fitted off-policy
value predictors. Each base learner estimates a \(Q\)-function and induces a
target-policy value predictor, which is then calibrated. We use linear and
random-feature FQE with linear, histogram, and isotonic calibrators. The
synthetic experiments isolate the failure mode in
Table~\ref{tab:rl-calibration-motivation}: informative but mis-scaled values.
We use 5-fold cross-calibration \citep{van2023causal}, except under temporal
shift, where FQE is trained on an old-regime batch and \(g\) is fit on a small
current-regime batch. All main results use 100 replications and independent
evaluation data; values below one indicate improvement. Details are in
Appendix~\ref{app:experiments}.

\begin{figure}[t]
  \centering
  \includegraphics[width=0.8\linewidth]{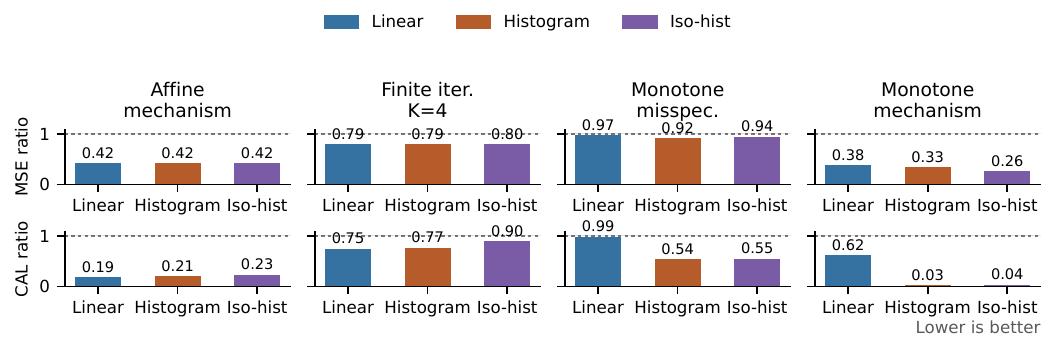}
  \vspace{-0.25em}
  \caption{\textbf{Post-hoc calibration improves mis-scaled value predictions.}
  Value MSE and Bellman calibration error ratios across synthetic failure
  modes, relative to the matched raw baseline.}
  \vspace{-0.5em}
  \label{fig:calibration-story}
\end{figure}

Figure~\ref{fig:calibration-story} reports ratios to the matched raw baseline.
Linear calibration corrects affine scale errors, while histogram and
isotonic-histogram calibration give the strongest gains under nonlinear
monotone misspecification. With finite Bellman-iteration error, all three
post-hoc maps reduce true-\(V\) MSE by about \(20\%\).
Table~\ref{tab:main-simulation-results} (Appendix~\ref{app:tablemain}) gives
the complementary audit, including temporal reward shift.

As an external check, we calibrate frozen fitted \(Q\)-functions on Deep OPE
Hopper-medium. Each \(Q\)-function is trained once and held fixed; only the
one-dimensional map \(g\) is fit on calibration trajectories, with evaluation on
disjoint diagnostic trajectories. We calibrate \(\widehat Q(S,A)\) on logged
state-action pairs, and each Bellman backup averages over target-policy next
actions. Figure~\ref{fig:hopper-q-calibration} summarizes a neural-FQE run with
1000 gradient updates, 20 random trajectory splits, and all eleven target
policies. Isotonic-histogram calibration reduces plug-in and debiased Bellman
calibration error to 0.787 and 0.771 of the raw \(Q\)-function, while
\(Q\)-Bellman MSE and scalar OPE error are nearly unchanged (1.002 and 0.978).
Thus, for a stronger frozen \(Q\)-function, calibration mainly improves Bellman
reliability rather than inducing artificial MSE gains.
Appendix~\ref{app:hopper-q-calibration} also reports a 100-update neural-FQE
stress test, where the deliberately undertrained \(Q\)-function leaves
substantial scale/projection error that post-hoc calibration can correct.

\section{Conclusion}
Bellman calibration provides an estimable, post-hoc reliability criterion:
predictions should agree, on average, with target-policy Bellman backups among
states assigned similar values. Iterated Bellman Calibration turns this
criterion into a one-dimensional recalibration procedure that controls
calibration error without Bellman completeness or value-function realizability.
Appendix~\ref{app:concl} discusses takeaways, coverage and nuisance-estimation
limitations, and future directions.

\bibliographystyle{plainnat}
\bibliography{ref}

\appendix

\section{Table and figure for experiment}
\label{app:tablemain}
\begin{table*}[h]
\centering
\caption{\textbf{Value-scale correction across offline RL failure modes.}
Ratios are relative to the matched uncalibrated baseline. Win rate is the
seed-matched true-\(V\) MSE win rate.}
\label{tab:main-simulation-results}
\scriptsize
\setlength{\tabcolsep}{4pt}
\begin{tabular}{@{}p{0.45\linewidth}lrrr@{}}
\toprule
Regime and learner & Calibrator & Relative \(V\) MSE & Relative cal. err. & Win rate \\
\midrule
Affine misspec., linear FQE & Linear & 0.750 & 0.631 & 0.94 \\
Affine misspec., linear FQE & Isotonic & 0.734 & 0.764 & 0.94 \\
Finite iteration, RF FQE \(K=4\) & Linear & 0.780 & 0.703 & 0.94 \\
Finite iteration, RF FQE \(K=4\) & Isotonic & 0.776 & 0.736 & 0.94 \\
Temporal reward shift, RF FQE & Linear recent & 0.630 & 0.717 & 0.97 \\
Temporal reward shift, RF FQE & Isotonic recent & 0.627 & 0.916 & 0.98 \\
Monotone misspec., linear FQE & Linear & 0.972 & 0.986 & 0.64 \\
Monotone misspec., linear FQE & Isotonic & 0.916 & 0.537 & 0.79 \\
Monotone misspec., linear FQE & Histogram & 0.917 & 0.544 & 0.77 \\
\bottomrule
\end{tabular}
\end{table*}

\begin{figure*}[h]
  \centering
  \includegraphics[width=0.98\textwidth]{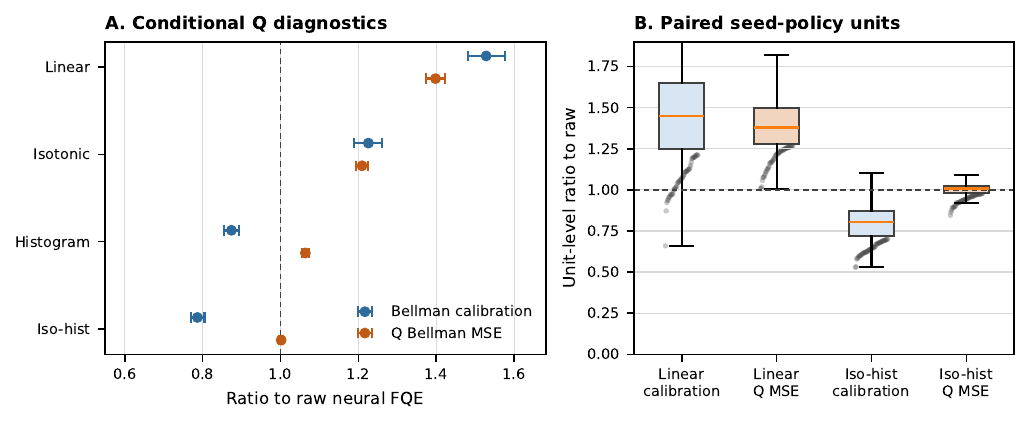}
  \caption{\textbf{Post-hoc calibration improves Bellman reliability for a
  frozen neural-FQE \(Q\)-function on Hopper-medium.}
  This benchmark uses a 1000-update neural-FQE \(Q\)-function, 20 random
  trajectory splits, and all eleven official Hopper-medium target policies.
  Panel A reports ratios to the matched raw \(Q\)-function; values below one
  improve on the frozen \(Q\)-function, and intervals are paired bootstrap
  confidence intervals clustered by seed-policy pair. Panel B shows paired
  unit-level ratios for the simple linear calibrator and isotonic-histogram
  calibration. In this more trained setting, calibration improves Bellman
  calibration error while leaving \(Q\)-Bellman MSE approximately unchanged.}
  \label{fig:hopper-q-calibration}
\end{figure*}

\paragraph{Interpretation of the Hopper result.}
Figure~\ref{fig:hopper-q-calibration} also illustrates the role of the
calibration class. Linear and isotonic calibration do not improve this
better-trained neural-FQE \(Q\)-function; in fact, both can worsen the Bellman
diagnostics. This is not a contradiction of the theory. Our guarantees apply to
histogram-style calibrators with a fixed score partition, for which the
calibrated class is closed under the fitted Bellman update on that partition.
A fixed affine map, and a globally monotone isotonic map of the scalar score
\(\widehat Q(S,A)\), do not have this fixed-partition Bellman-completeness
property. After the target-policy Bellman backup, the induced correction may
fall outside these calibration classes, producing residual Bellman
misspecification. Histogram and isotonic-histogram calibration satisfy the
fixed-partition structure used in the theory. Consistent with this, the
histogram-style calibrators give the reliable Hopper gains: they reduce Bellman
calibration error for the frozen \(Q\)-function, while scalar OPE error and
\(Q\)-Bellman MSE remain approximately neutral.

\section{Conclusion and Future Work}
\label{app:concl}

We introduced Bellman calibration as a post-hoc reliability criterion for
long-horizon value prediction. The criterion asks that predictions agree, on
average, with their target-policy Bellman backups within groups of states
assigned similar predicted values. This is weaker than statewise Bellman
consistency, but strong enough to diagnose systematic value-scale errors and to
support model-agnostic recalibration of an already learned value predictor.

Iterated Bellman Calibration turns this criterion into a practical correction:
fit a one-dimensional calibration map using off-policy Bellman targets, then
iterate the induced calibrated Bellman update. Because the recalibration is
one-dimensional, histogram and isotonic variants can control Bellman calibration
error without Bellman completeness or value-function realizability. The
value-error bounds separate statistical calibration error, finite-iteration
error, nuisance estimation error, and the approximation error of the best
predictor expressible as a function of the original value estimate. Thus
calibration is most useful when the base predictor is informative but
mis-scaled, and its limits are explicit when the original prediction discards
state information needed for accurate value prediction.

The main practical challenges are coverage, nuisance estimation, and
extrapolation. Off-policy Bellman calibration requires enough overlap to
estimate target-policy Bellman backups, and poor coverage can make
importance-weighted or doubly robust targets unstable. In limited-coverage
settings, it can be preferable to calibrate an action-value function directly:
treat \(X=(S,A)\) as the prediction input, fit the calibration map to
\(\widehat Q(S,A)\), and average only over target-policy next actions in the
backup. This avoids the additional propensity or density-ratio weighting needed
to calibrate an implied value function under the target-policy state
distribution. The resulting guarantee is a \(Q\)-calibration guarantee, not a
scalar policy-value guarantee: it assesses and corrects the frozen
\(Q\)-function on the logged state-action distribution before any aggregation
into an OPE estimate.

Future work includes combining Bellman calibration with stationary-density-ratio
or emphatic-style weighting
\citep{vdLaanKallus2025FQE,mahmood2015emphatic,patterson2022generalized},
developing smoother calibrators for extrapolation under distribution shift
\citep{platt1999probabilistic,jiang2011smooth,guo2017calibration}, and studying
calibration as both a diagnostic and post-hoc correction in larger offline RL
benchmarks.

\section{Additional Experimental Details}
\label{app:experiments}

For the non-temporal-shift experiments, we use 5-fold cross-calibration
\citep{van2023causal}. In each fold, the value learner and calibrator are
trained out of fold, and calibrated predictions are evaluated on the held-out
fold. Fold-specific calibrated predictors are aggregated by pointwise medians.
The uncalibrated comparator uses the same base learner refit on the full
training sample. Calibration is fit by importance-weighted fitted value
iteration, using the target-to-behavior action-probability ratio as weights. In
the temporal-shift experiment, we train FQE on a large old reward-regime batch,
fit the calibration map on a small held-out current-regime batch, and evaluate
only on current-regime states. All main results use 100 independent replications
and independent evaluation data.

\paragraph{Data-generating processes.}
All four simulation panels use discounted MDPs with continuous states
\(S_t\in\mathbb{R}^d\), actions \(A_t\in\{1,2,3\}\), and softmax behavior and
target policies. The target policy has logits \(s^\top w_a^\pi/\tau_\pi\). The
behavior policy uses shifted logits \(s^\top w_a^b/\tau_b\), where
\(w_a^b=w_a^\pi-\Delta u_a\). The main experiments use the moderate-shift
setting. Initial states are Gaussian, \(S_0\sim N(\mu_0,I_d)\). Except where
noted below, transitions and rewards are
\[
  S_{t+1}
  =
  \tanh\{M S_t+\alpha_{A_t}+0.22\sin(S_t)+0.08\cos(S_t^{\mathrm{rev}})\}
  + \sigma_\varepsilon \varepsilon_t ,
\]
and
\[
  r_0(s,a)
  =
  s^\top\beta
  +0.65\sin(s_1)
  +0.35\cos(s_2)
  +s^\top\eta_a
  +\rho_a
  -0.06\|s\|_2^2 .
\]
Observed rewards add mean-zero Gaussian noise. The ground-truth value \(V^\pi(s)\) is
computed by independent target-policy Monte Carlo rollouts using the conditional
mean reward. Training transitions, calibration folds, Bellman-evaluation
transitions, initial evaluation states, and Monte Carlo rollouts use independent
seeds.

The four panels isolate different value-scale errors. In the affine
misspecification panel, the same nonlinear MDP is used, but the reward is
transformed to
\[
  r_{\mathrm{aff}}(s,a)=1.25+0.7\,r_0(s,a).
\]
The learner is linear FQE with a deliberately restricted feature class. This
produces informative value predictions with an affine slope/intercept error, so
linear calibration should be sufficient.

In the finite-iteration panel, the dynamics are made persistent,
\[
  S_{t+1}=0.92S_t+\alpha_{A_t}+0.03\sin(S_t)+\sigma_\varepsilon\varepsilon_t,
\]
with \(\gamma=0.985\) and horizon \(120\). Rewards are smooth but long-horizon:
\[
  r_{\mathrm{fin}}(s,a)
  =
  0.8+0.9\tanh(s_1)+0.25s_2+\rho_a
  +0.15\tanh(s_1+0.5s_2)+s^\top\eta_a .
\]
Random-feature FQE is stopped after \(K=4\) Bellman iterations, so its values
retain useful ranking information but are under-scaled relative to the
long-horizon target.

In the temporal reward-shift panel, the old and current environments share
dynamics, policies, and state distribution. Only the reward changes:
\[
  r_{\mathrm{current}}(s,a)=2.0+1.5\,r_{\mathrm{old}}(s,a).
\]
We train random-feature FQE on a large old-regime behavior batch
(\(n_{\mathrm{old}}=2000\)), fit the calibration map on a small recent
current-regime batch (\(n_{\mathrm{recent}}=100\)), and evaluate on independent
current-regime ground-truth values and Bellman-calibration transitions. A same-size
current-regime retrain comparator uses the same learner.

In the monotone-misspecification panel, the nonlinear reward is transformed by
a monotone saturation,
\[
  r_{\mathrm{mono}}(s,a)
  =
  \operatorname{sign}\{r_0(s,a)\}\sqrt{|r_0(s,a)|+1}.
\]
The linear FQE value prediction remains sufficiently ordered to be useful but has a
nonlinear value-scale distortion. This is the setting where isotonic and
histogram calibration should improve over an affine correction.

\paragraph{Baseline value predictors.}
Each method first fits a \(Q\)-function from the offline behavior-policy data
and then forms
\[
  \widehat V(s)=\sum_a \pi(a\mid s)\widehat Q(s,a).
\]
The experiments use random-feature FQE and linear FQE because they give
transparent error mechanisms: finite Bellman-iteration bias and affine or
monotone value-scale misspecification.

\paragraph{Cross-calibration.}
For fold \(k\), the \(Q\)-learner is fit on the other folds. The held-out fold
supplies \((\widehat V^{(-k)}(S_i),\widehat V^{(-k)}(S_i'),R_i)\). These
out-of-fold tuples are pooled to fit a single value-space calibration map by
importance-weighted fitted value iteration. At evaluation time, calibrated
values are aggregated as
\[
  \widehat V_{\mathrm{cal}}(s)
  =
  \mathrm{median}_{k=1,\ldots,K_{\mathrm{fold}}}
  g\!\left(\widehat V^{(-k)}(s)\right).
\]
The matched uncalibrated comparator is a \(Q\)-learner refit on all available
training data.

\paragraph{Value-space fitted value iteration.}
Let \(\omega_i=\pi(A_i\mid S_i)/b(A_i\mid S_i)\). In all reported runs the
action-ratio weights are clipped at 20 and normalized to mean one within each
calibration fit. Starting from the identity map, calibration iterates
\[
  g_{t+1}\in\arg\min_g\sum_i \omega_i
  \left[g(\widehat V(S_i))-\{R_i+\gamma g_t(\widehat V(S_i'))\}\right]^2.
\]
The implemented calibrators are weighted affine regression, weighted histogram
binning, weighted isotonic regression, and a weighted isotonic-histogram
hybrid.

\paragraph{Metrics.}
The scalar OPE estimate is the mean of \(g(\widehat V(S_0))\) over independent
initial states. The main prediction metric is held-out value-function MSE,
\[
  \mathbb{E}\!\left[\{g(\widehat V(S))-V^\pi(S)\}^2\right],
\]
estimated using independent target-policy ground-truth values. Bellman calibration
error is estimated on independent Bellman-evaluation transitions by binning
\(g(\widehat V(S))\) into quantile bins and averaging the weighted squared gaps
between bin-mean predictions and bin-mean Bellman outcomes
\[
  r_0(S,A)+\gamma g(\widehat V(S')) .
\]
We use independent evaluation batches of size 50,000 and 50 bins.

\section[Hopper Q-Function Calibration Benchmark]{Hopper \(Q\)-Function Calibration Benchmark}
\label{app:hopper-q-calibration}

\paragraph{Purpose and data.}
The Hopper experiment asks whether post-hoc calibration also helps fitted
\(Q\)-functions in a real offline-RL benchmark. We use
\texttt{hopper-medium-v0} from the Deep OPE benchmark, together with all eleven
official Hopper-medium target policies. The reported run uses 20 random
trajectory splits, eleven target policies, and 220 completed
neural-FQE policy--seed fits. For each seed, every calibrator uses the
same trajectory split: 60\% for training the base \(Q\)-function, 20\% for
fitting the post-hoc calibration map, and 20\% for held-out diagnostics.

The comparison is controlled in three ways. First, each base \(Q\)-function is
trained only once, using only the training trajectories, and is then held fixed.
Second, all calibrators are fit on the same calibration trajectories and
evaluated on the same disjoint diagnostic trajectories. Third, no official
policy-return labels are used to train the \(Q\)-functions or to fit the
calibration maps; the official returns are used only for the secondary scalar
OPE diagnostic.

\paragraph{\(Q\)-function calibration protocol.}
Each base learner estimates a \(Q\)-function \(\widehat Q(s,a)\) from the
training trajectories. Calibration is then fit as a one-dimensional correction
to the scalar prediction \(\widehat Q(S_i,A_i)\) on logged state-action pairs.
For a candidate map \(g\), the Bellman target is
\[
  Y_i(g)=R_i+\gamma\,\mathbb E_{A'\sim\pi(\cdot\mid S_i')}
  g\{\widehat Q(S_i',A')\}.
\]
This is the direct \(Q\)-function analogue of value calibration. Value
calibration groups states by their predicted values \(\widehat V(S)\);
\(Q\)-function calibration groups logged state-action pairs by their predicted
values \(\widehat Q(S,A)\), then uses the target policy to average over the next
action \(A'\). After fitting \(g\), scalar OPE estimates can be formed by
averaging \(g\{\widehat Q(S_0,A_0)\}\) over target-policy initial actions, but
those scalar estimates are not used for training or calibration.

\paragraph{Why calibrate \(Q\) rather than the implied value.}
One could first average the fitted \(Q\)-function into an implied value
\(\widehat V(S)=E_{A\sim\pi(\cdot\mid S)}\widehat Q(S,A)\) and then calibrate
\(\widehat V\). In this benchmark, that would require evaluating Bellman
targets under the target-policy state distribution, which would introduce
propensity or density-ratio weighting from the logged policy to the target
policy. Hopper-medium has limited overlap for several target policies, so those
weights can be noisy.

The \(Q\)-function version avoids that extra source of variance. We use the
logged pair \((S_i,A_i)\) as the calibration input, score it by the single
number \(\widehat Q(S_i,A_i)\), and use the target policy only inside the next
action average \(A'\sim\pi(\cdot\mid S_i')\). Equivalently, this is value
calibration applied to the Markov chain whose state is the pair \(X=(S,A)\) and
whose next state is \((S',A')\). The calibrated Bellman target is
\[
  R+\gamma E_{A'\sim\pi(\cdot\mid S')}g\{\widehat Q(S',A')\}.
\]
Thus the same one-dimensional histogram, isotonic, and linear calibration maps
apply, but the diagnostic distribution is the logged state-action distribution.
Scalar OPE from \(g\{\widehat Q\}\) is reported only after calibration, as a
secondary check.

\paragraph{Base learner, calibrators, and tuning.}
The Hopper benchmark uses neural FQE as the nonlinear \(Q\)-function learner.
We report two frozen neural-FQE settings. The more trained setting uses a
two-layer 256-by-256 network, 1000 gradient updates, learning rate \(10^{-3}\),
target-network rate 0.005, and batch size 256. The lightweight stress-test
setting uses the same architecture with 100 gradient updates and learning rate
\(3\times10^{-4}\). Every frozen \(Q\)-function is evaluated with the same
post-hoc calibrator set: none, linear, isotonic, histogram, and
isotonic-histogram. There is no per-seed or per-policy retuning.

\paragraph{Metrics and weighting.}
The primary metrics are computed on the held-out diagnostic trajectories:
plug-in Bellman calibration error, debiased Bellman calibration error, and
\(Q\)-Bellman MSE. The reported calibration metrics use no propensity or
importance weighting; all final rows have \texttt{weighting = none}. Scalar OPE
absolute error against official policy returns is reported only as a secondary
sanity check. Uncertainty is computed by paired bootstrap intervals clustered by
\((\mathrm{seed},\mathrm{policy})\), and win rates compare each calibrated
method to its matched raw \(Q\)-function on the same seed-policy pair.

\paragraph{More trained neural FQE.}
With 1000 gradient updates, isotonic-histogram calibration reduces plug-in and
debiased Bellman calibration error to 0.787 and 0.771 of the raw fitted
\(Q\)-function. The corresponding \(Q\)-Bellman MSE is essentially unchanged
(1.002 of raw), and scalar OPE error is also roughly neutral (0.978). Histogram
calibration gives smaller calibration gains (0.874 plug-in, 0.860 debiased)
with a modest \(Q\)-MSE increase (1.064). Thus, once the neural FQE learner is
trained more extensively, the evidence is a reliability gain in Bellman
calibration rather than a large additional MSE gain.

\paragraph{Full audit.}
Figure~\ref{fig:hopper-q-audit} reports the lightweight 100-update neural-FQE
stress test. In this deliberately undertrained setting, calibration improves
the primary \(Q\)-Bellman diagnostics across linear, isotonic, histogram, and
isotonic-histogram maps, with especially large \(Q\)-MSE reductions. The
histogram-based maps are more flexible and can require careful bin tuning:
well-chosen bins can give effective post-hoc corrections, whereas overly fine
bins can produce noisy calibration maps when the calibration sample is limited.
The linear calibrator is therefore an important low-tuning baseline, and it
already gives large gains in this stress test.

\begin{figure*}[t]
  \centering
  \includegraphics[width=0.98\textwidth]{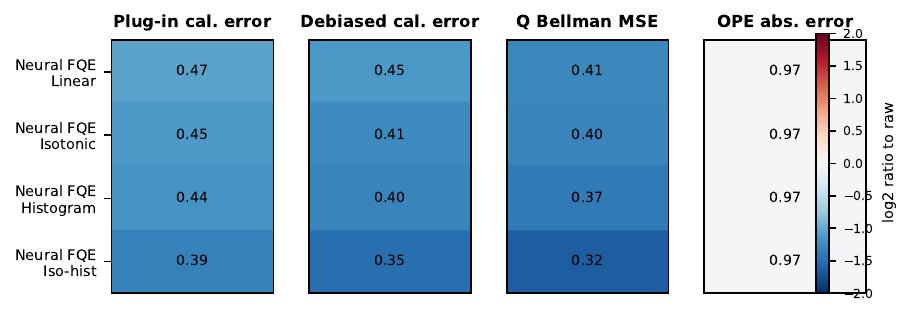}
  \caption{\textbf{Hopper lightweight neural-FQE \(Q\)-function calibration audit.}
  Entries are ratios to the matched raw \(Q\)-function. Values below one improve
  on the raw baseline, while values above one are regressions. This appendix
  audit uses the undertrained 100-update neural-FQE stress test and includes
  scalar OPE only as a secondary check.}
  \label{fig:hopper-q-audit}
\end{figure*}

\paragraph{Reproducibility.}
Both Hopper neural-FQE runs use \texttt{hopper-medium-v0}, all eleven official
Hopper-medium policies, 20 random trajectory splits, and 220 completed
policy--seed fits. For every seed and policy, the same 60/20/20 trajectory
split is reused across all calibrators, and the official policy returns are
used only for the secondary scalar OPE diagnostic.

\section{Additional details and extensions}

\subsection{Hybrid Isotonic--Histogram Calibration}
\label{appendix::hybrid}

Isotonic iterated Bellman calibration adaptively pools nearby predicted values
through the monotonicity constraint. This adaptivity is useful in practice, but
it complicates the estimation-error analysis because applying
Algorithm~\ref{alg::cal-gen} with \(\mathcal G=\mathcal G_{\mathrm{iso}}\) may
produce a different data-dependent partition at each iteration. The iterates
therefore need not correspond to a single aggregated Bellman operator.

The hybrid isotonic--histogram procedure separates partition selection from
Bellman iteration. It first performs one isotonic regression of the initial
Bellman target as a function of \(\hat v(S)\), extracts the flat regions of the
fitted isotonic function as bins, and then runs histogram-based Bellman
calibration on this fixed partition. Thus the method retains the adaptivity of
isotonic calibration while restoring the fixed aggregated-operator structure
used in the histogram analysis.

\begin{algorithm}[t]
\caption{\textbf{Iso--Hist Iterated Bellman Calibration}}
\label{alg::iso-hist}
\begin{algorithmic}[1]

\BCINPUT Value predictor $\hat{v}$, calibration data $\mathcal C_n$,
Bellman target-estimate map $v\mapsto\widehat Y_v$, iterations $K$

\STATE {\small\textbf{Stage 1: fit outcome-adaptive partition}}
\STATE \hspace{1em} $\widehat Y_i^{(0)} \gets \widehat Y_{\hat v}(O_i)$
\STATE \hspace{1em} $g_{\mathrm{iso}} \gets
       \arg\min_{g \in \mathcal G_{\mathrm{iso}}}
       \sum_{i=1}^n (\widehat Y_i^{(0)} - g(\hat{v}(S_i)))^2$
\STATE \hspace{1em} Extract bins $\hat I_1,\dots,\hat I_{\hat B}$
       from flat regions of $g_{\mathrm{iso}}$
\STATE \hspace{1em} Define $\widehat{\mathcal G}_{\hat B} :=
       \{g:g  \text{ is constant on each }\hat I_b\}$

\STATE {\small\textbf{Stage 2: histogram calibration}}
\STATE \hspace{1em} Apply Algorithm~\ref{alg::cal-gen} with class
       $\widehat{\mathcal G}_{\hat B}$ to obtain $V_K$

\BCOUTPUT $V_K$

\end{algorithmic}
\end{algorithm}

Once the partition is fixed, the guarantees of Section~\ref{sec:hist} apply
directly, provided the number of bins produced by the initial isotonic step
admits a deterministic high-probability bound. More generally, the same
argument applies to other data-adaptive binning methods that produce a fixed
partition before the Bellman iterations begin.

\begin{theorem}[Calibration and Estimation Error for Alg.~\ref{alg::iso-hist}]
\label{theorem::regretiso}
Let $V_K$ denote the output of Algorithm~\ref{alg::iso-hist}. Suppose
the isotonic regression step produces at most $\hat B$ bins, where
$\hat B \le B_n$ for some deterministic sequence $B_n$ with probability at least $1 - \delta$. Assume \ref{cond::C1}--\ref{cond::Ccoverage} and \ref{cond::C3} with $B = B_n$.
Then, with probability at least $1 - 2\delta$, the guarantees in
Theorems~\ref{theorem::Calerrorhist} and~\ref{theorem::regretHist} hold for
$V_K$ with $B := B_n$.
\end{theorem}

Existing isotonic regression theory suggests that the effective number of
induced bins satisfies \(\widehat B=O(n^{1/3})\) when the calibration function
is continuously differentiable \citep{deng2021confidence}; this rate can also
be enforced by directly constraining the number of constant segments
\citep{van2024self}. For the approximation term in
Theorem~\ref{theorem::regretHist} to remain comparable to the initial error,
the induced histogram class must also approximate the identity map on the range
of \(\hat v\). This is plausible when \(\hat v\) is already close to \(v_0\);
when \(\hat v\) is substantially miscalibrated, isotonic regression may instead
pool poorly predicted regions into a common calibrated value.

\subsection{Alternative Calibration Targets}
\label{app:alternative-calibration-targets}

The Bellman calibration criterion studied in the main text calibrates a
candidate value predictor \(v\) with respect to its own Bellman target,
\[
E_\pi\!\left[R+\gamma v(S')\mid v(S)\right].
\]
A stronger alternative would instead calibrate \(v\) directly against the true
value function \(v_0\). Specifically, we say that \(v\) is
\emph{strongly Bellman calibrated} if
\[
v(S)
=
E\!\left[v_0(S)\mid v(S)\right]
\quad\text{almost surely}.
\]
Equivalently, among states assigned the same predicted value, the prediction
agrees on average with the true long-run return under \(\pi\).

This target is conceptually appealing, but substantially less practical. It
requires estimating the conditional mean of the unknown quantity \(v_0(S)\)
given the predicted value \(v(S)\). In general, this conditional mean is not directly
available from one-step transitions alone and is typically no easier to estimate
than the underlying value function itself
\citep{kallus2020double,kallus2022efficiently}. We therefore focus on the
weaker but directly attainable Bellman-target calibration criterion developed
in the main text.

\section{Stationary Measures and Contraction Results}
\label{appendix::stationary}

\begin{lemma}[Bellman contraction under a stationary measure]
\label{lemma::contraction}
Let $P$ be a Markov operator with stationary distribution $\mu$ (i.e., $\mu=\mu P$),
and define the Bellman operator $\mathcal T f := r + \gamma P f$.
Then $P$ is nonexpansive in $L^2(\mu)$:
\[
\|P(f-g)\|_{2,\mu} \;\le\; \|f-g\|_{2,\mu}.
\]
Consequently, $\mathcal T$ is a $\gamma$-contraction:
\[
\|\mathcal T (f-g)\|_{2,\mu}
\;\le\;
\gamma\,\|f-g\|_{2,\mu}.
\]
\end{lemma}
\begin{proof}
Because $\mu$ is stationary for $P$, the operator $P$ is nonexpansive in
$L^2(\mu)$. Jensen's inequality implies
\[
(P h)^2 \le P(h^2) \qquad\text{pointwise}.
\]
Integrating both sides with respect to $\mu$ and using $\mu = \mu P$ yields
\[
\|P h\|_{2,\mu}^2
= \int (P h)^2 \, d\mu
\le \int P(h^2)\, d\mu
= \int h^2\, d\mu
= \|h\|_{2,\mu}^2.
\]
Now apply this with $h = f - g$. Since the reward function $r$ cancels,
\[
\mathcal{T}f - \mathcal{T}g
= \gamma\, P(f-g).
\]
Therefore,
\[
\|\mathcal{T} f - \mathcal{T} g\|_{2,\mu}
= \gamma \|P(f-g)\|_{2,\mu}
\le \gamma \|f-g\|_{2,\mu},
\]
which proves the claim.
\end{proof}

\begin{lemma}[Stationarity under coarsening]
\label{lemma:aggregated-stationary}
Let $P$ be a Markov kernel on $\mathcal S$ and let $\lambda$ be stationary for $P$
(i.e., $\lambda P = \lambda$).
For any measurable coarsening map $g:\mathcal S\to\mathbb R$, define
\[
P_g f(s)
:=
\mathbb{E}[\,f(S') \mid g(S)=g(s)\,],
\]
where $(S,S')$ has joint law $\lambda(ds)\,P(s,ds')$.
Then $\lambda$ is stationary for $P_g$.
\end{lemma}

\begin{proof}
Let $(S,S')\sim \lambda(ds)\,P(s,ds')$, so $S\sim\lambda$ and $S'\mid S\sim P(S,\cdot)$.
Then
\begin{align*}
\int P_g f(s)\,\lambda(ds)
&=
\mathbb{E}_{\lambda,P}\!\left[P_g f(S)\right] \\[0.3em]
&=
\mathbb{E}_{\lambda,P}\!\left[\mathbb{E}_{\lambda,P}\!\left[f(S') \mid g(S)\right]\right] \\[0.3em]
&=
\mathbb{E}_{\lambda,P}[f(S')].
\end{align*}
Since $\lambda$ is stationary for $P$, the marginal of $S'$ is again $\lambda$, hence
$\mathbb{E}_{\lambda,P}[f(S')] = \int f(s)\,\lambda(ds)$.
\end{proof}

\section{Calibration--Refinement Proofs}

In this section, set
\[
\mathcal B_{\pi,\hat v}q
:=
\Pi_{\hat v}\mathcal T_\pi q,
\qquad
V_{\hat v}^\star\in\mathcal V_{\hat v},
\qquad
V_{\hat v}^\star
=
\mathcal B_{\pi,\hat v}V_{\hat v}^\star .
\]

\begin{theorem}[Sharper calibration--refinement bound]
\label{theorem::calrefineOracle}
Under \ref{cond::A1},
\[
\|\hat v-v_0\|_{2,\rho}
\le
\|\hat v-V_{\hat v}^\star\|_{2,\rho}
+
\frac{\sqrt{\mathcal C_{\hat v,\gamma}}}{1-\gamma}
\min_{\theta:\mathbb R\to\mathbb R}
\|\theta\circ\hat v-v_0\|_{2,\rho}.
\]
\end{theorem}

\begin{proof}[Proof of Theorem \ref{theorem::calrefineOracle}]
Define \(P_{\pi,\hat v} := \Pi_{\hat v}P_{\pi}\) and
\(r_{\pi,\hat{v}} := \Pi_{\hat v} r_{\pi}\). Because \(v_0\) and
\(V_{\hat v}^\star\) satisfy the fixed-point equations
\[
v_0 = \mathcal T_\pi v_0,
\qquad
V_{\hat v}^\star = \mathcal B_{\pi,\hat v} V_{\hat v}^\star
= r_{\pi,\hat{v}} + \gamma\,P_{\pi,\hat v}V_{\hat v}^\star,
\]
we begin by writing
\[
V_{\hat v}^\star - v_0
=
\bigl( \mathcal B_{\pi,\hat v}V_{\hat v}^\star
      - \mathcal B_{\pi,\hat v}v_0 \bigr)
+
\bigl( \mathcal B_{\pi,\hat v}v_0 - \mathcal{T}_\pi v_0 \bigr).
\]
Using linearity on differences,
\[
\mathcal B_{\pi,\hat v}V_{\hat v}^\star
-
\mathcal B_{\pi,\hat v}v_0
=
\gamma\,P_{\pi,\hat v}(V_{\hat v}^\star - v_0),
\]
so
\[
V_{\hat v}^\star - v_0
=
\gamma\,P_{\pi,\hat v}(V_{\hat v}^\star - v_0)
+
\bigl(\mathcal B_{\pi,\hat v}v_0 - v_0\bigr),
\]
where we used \(\mathcal T_\pi v_0 = v_0\) for the second term. Since
\(\mathcal B_{\pi,\hat v}v_0=\Pi_{\hat v}v_0\), this gives
\[
V_{\hat v}^\star-v_0
=
\gamma P_{\pi,\hat v}(V_{\hat v}^\star-v_0)
+
(\Pi_{\hat v}-I)v_0 .
\]
Let \(a:=(\Pi_{\hat v}-I)v_0\). Iterating the display yields
\[
V_{\hat v}^\star-v_0
=
\sum_{t=0}^\infty \gamma^t P_{\pi,\hat v}^t a .
\]
For each \(t\ge0\), set
\[
c_t
:=
\left\|
\frac{d(\rho P_{\pi,\hat v}^t)}{d\rho}
\right\|_\infty .
\]
Jensen's inequality implies
\[
\|P_{\pi,\hat v}^t a\|_{2,\rho}
\le
\sqrt{c_t}\|a\|_{2,\rho}.
\]
Therefore,
\[
\|V_{\hat v}^\star-v_0\|_{2,\rho}
\le
\frac{\sqrt{\mathcal C_{\hat v,\gamma}}}{1-\gamma}
\|\Pi_{\hat v}v_0-v_0\|_{2,\rho}.
\tag{$\star$}
\]
Since \(\Pi_{\hat v}\) is the \(L^2(\rho)\)-projection onto
\(\{\theta\circ\hat v:\theta:\mathbb R\to\mathbb R\}\),
\[
\|\Pi_{\hat v}v_0-v_0\|_{2,\rho}
=
\min_{\theta:\mathbb R\to\mathbb R}
\|\theta\circ\hat v-v_0\|_{2,\rho}.
\]
Finally, by the triangle inequality,
\[
\|\hat{v} - v_0\|_{2,\rho}
\le
\|\hat{v} - V_{\hat v}^\star\|_{2,\rho}
+
\|V_{\hat v}^\star - v_0\|_{2,\rho},
\]
Combining this display with \((\star)\) yields the stated result.
\end{proof}

\begin{proof}[Proof of Theorem \ref{theorem::calrefine}]
Let \(P_{\pi,\hat v} := \Pi_{\hat v}P_{\pi}\).
As in the proof of Theorem~\ref{theorem::calrefineOracle},
\[
\|V_{\hat v}^\star-v_0\|_{2,\rho}
\le
\frac{\sqrt{\mathcal C_{\hat v,\gamma}}}{1-\gamma}
\min_{\theta:\mathbb R\to\mathbb R}
\|\theta\circ\hat v-v_0\|_{2,\rho}.
\tag{$\star$}
\]
Moreover,
\[
\hat v - V_{\hat v}^\star
=
\hat v-\mathcal B_{\pi,\hat v}\hat v
+
\gamma P_{\pi,\hat v}(\hat v-V_{\hat v}^\star),
\]
and therefore
\[
\hat v - V_{\hat v}^\star
=
\sum_{t=0}^\infty \gamma^t
P_{\pi,\hat v}^t
\{\hat v-\mathcal B_{\pi,\hat v}\hat v\}.
\]
Since \(\mathcal B_{\pi,\hat v}\hat v=\Gamma_0(\hat v)\), the
concentrability argument above yields
\[
\|\hat v - V_{\hat v}^\star\|_{2,\rho}
\le
\frac{\sqrt{\mathcal C_{\hat v,\gamma}}}{1-\gamma}
\mathrm{Cal}_{\ell^2}(\hat v).
\tag{$\dagger$}
\]
The triangle inequality and displays \((\star)\) and \((\dagger)\) give the
claim.
\end{proof}

\section{Proof of Theorem \ref{theorem::DRpseudo}}

\begin{proof}[Proof of Theorem \ref{theorem::DRpseudo}]
Fix $v$ and treat the nuisance estimates as fixed functions. By definition,
\[
\widehat{\mathcal T}_\pi v(s)
=
\mathbb{E}\Big[
(\pi \hat q_v)(S)
+
\widehat w_{\pi}(A \mid S)
\big\{
  R + \gamma\, v(S') - \hat{q}_v(S,A)
\big\}
\;\Big|\; S = s
\Big].
\]
Since $(\pi \hat q_v)(S)$ depends only on $S$,
\[
\mathbb{E}\big[(\pi \hat q_v)(S)\mid S=s\big] = (\pi \hat q_v)(s).
\]
For the second term, using the definition of $b_0$ and $q_v := r_0 + \gamma P v$,
\begin{align*}
&\mathbb{E}\Big[
\widehat w_{\pi}(A \mid S)
\big\{
  R + \gamma\, v(S') - \hat{q}_v(S,A)
\big\}
\;\Big|\; S = s
\Big] \\
&\quad=
b_0\Big\{
\widehat w_\pi(\cdot \mid s)\,
\mathbb{E}\big[
R + \gamma\, v(S') - \hat q_v(s,\cdot)
\mid S = s, A = \cdot
\big]
\Big\}(s) \\
&\quad=
b_0\big\{\widehat w_\pi (q_v - \hat q_v)\big\}(s).
\end{align*}
Hence
\[
\widehat{\mathcal T}_\pi v(s)
=
(\pi \hat q_v)(s)
+
b_0\big\{\widehat w_\pi (q_v - \hat q_v)\big\}(s).
\]

Subtracting $(\mathcal T_\pi v)(s) = (\pi q_v)(s)$ yields
\begin{align*}
\widehat{\mathcal T}_\pi v(s) - (\mathcal T_\pi v)(s)
&=
(\pi \hat q_v - \pi q_v)(s)
+
b_0\big\{\widehat w_\pi (q_v - \hat q_v)\big\}(s) \\
&=
(\pi (\hat q_v - q_v))(s)
-
b_0\big\{\widehat w_\pi (\hat q_v - q_v)\big\}(s).
\end{align*}
Using $(\pi f)(s) = b_0\{w_\pi f\}(s)$ with $f = \hat q_v - q_v$ gives
\[
(\pi (\hat q_v - q_v))(s)
=
b_0\{w_\pi (\hat q_v - q_v)\}(s),
\]
so
\[
\widehat{\mathcal T}_\pi v(s) - (\mathcal T_\pi v)(s)
=
b_0\big\{(w_\pi - \widehat w_\pi)(\hat q_v - q_v)\big\}(s).
\]
Taking conditional expectations given \(v(S)=t\) proves the claim.
\end{proof}

\section{Notation and Maximal Inequalities}

Let $P_0$ denote the joint distribution of $(S, A, R, S')$ induced by the
behavior policy, and let $P_n$ denote the empirical measure of the calibration
sample $\mathcal{C}_n$.

We define empirical \(L^2\) norms with respect to the state and next-state samples.  For any state function \(f\), let
\[
\|f\|_{n,S}
:=
\Big( \tfrac{1}{n} \sum_{i=1}^{n} f(S_i)^2 \Big)^{1/2},
\quad
\|f\|_{n,S'}
:=
\Big( \tfrac{1}{n} \sum_{i=1}^{n} f(S_i')^2 \Big)^{1/2},
\]
where \(S_i\) and \(S_i'\) denote the observed states and next states, respectively, in the calibration sample \(\mathcal{C}_n\).

Further define
\[
\widehat{\mathcal{G}}
:=
\Big\{ (f_1 - f_2)\big(\widehat Y_{f_2} - f_2\big)
:\;
f_1, f_2 \in \mathcal H_B(\hat v) \Big\}.
\]
By assumption, both \(\hat{v}\) and the Bellman target-estimate map
\(v\mapsto\widehat Y_v\) are fixed conditional on the training data, which is
independent of the calibration sample \(\mathcal{C}_n\). Consequently, the
classes \(\mathcal H_B(\hat v)\) and \(\widehat{\mathcal{G}}\) are
non-random conditional on the training dataset.

For any distribution \(Q\) and any uniformly bounded function class \(\mathcal{F}\), let
\(N(\varepsilon, \mathcal{F}, L^2(Q))\) denote the \(\varepsilon\)-covering number of \(\mathcal{F}\) under the \(L^2(Q)\) norm \citep{van1996weak}.
Define the uniform entropy integral of \(\mathcal{F}\) by
\begin{equation*}
\mathcal{J}(\delta, \mathcal{F})
:=
\int_{0}^{\delta}
\sup_{Q}
\sqrt{\log N(\epsilon, \mathcal{F}, L^2(Q))}\, d\epsilon ,
\end{equation*}
where the supremum is taken over all discrete probability distributions \(Q\).

Finally, for two quantities \(x\) and \(y\), we write \(x \lessapprox y\) to mean that \(x\) is bounded above by \(y\) up to a universal constant that depends only on global constants appearing in our conditions.

\subsection{Local maximal inequality}

 Let $O_1,\ldots,O_n \in \mathcal{O}$ be independent random variables. For any function $f:\mathcal{O} \to \mathbb{R}$, define
\begin{align}
   \|f\|_{2,\bar P} := \sqrt{\frac{1}{n} \sum_{i=1}^n\mathbb{E}[f(O_i)^2]}.
\end{align}

For a star-shaped class of functions $\mathcal{F}$ and a radius $\delta \in (0,\infty)$,
define the localized Rademacher complexity
\[
\mathcal{R}_n(\mathcal{F}, \delta)
:=
\mathbb{E}\left[
\sup_{\substack{f \in \mathcal{F} \\\|f\|_{2,\bar P} \le \delta}}
\frac{1}{n} \sum_{i=1}^n \epsilon_i f(O_i)
\right],
\]
where $\epsilon_i$ are i.i.d.\ Rademacher random variables.

The following lemma provides a local maximal inequality and restates Lemma~11 of \cite{foster2023orthogonal}.

\begin{lemma}[Local maximal inequality]\label{lemma:loc_max_ineq}
Let $\mathcal{F}$ be a star-shaped class of functions satisfying
$\sup_{f\in\mathcal{F}} \|f\|_{\infty} \le M$.
Let $\delta > 0$ satisfy the critical radius condition
$\mathcal{R}_n(\mathcal{F},\delta) \le \delta^2$.
Suppose further that
$n^{-1/2}\sqrt{\log\log(1/\delta)} = o(\delta)$. Then there exists a universal constant $C>0$ such that, for all $u \ge 1$, with
probability at least $1 - e^{-u^2}$, every $f \in \mathcal{F}$ satisfies
\[
\frac{1}{n} \sum_{i=1}^n
\bigl(f(O_i) - \mathbb{E}[f(O_i)]\bigr)
\;\le\;
C\Bigl(
\delta^2
+
\delta\|f\|_{2,\bar P}
+
\frac{u\|f\|_{2,\bar P}}{\sqrt{n}}
+
\frac{M u^2}{n}
\Bigr).
\]
\end{lemma}
\begin{proof}
This is Lemma~11 of \citet{foster2023orthogonal}, specialized to the notation
above.
\end{proof}

The following lemma bounds the localized Rademacher complexity in terms of the uniform entropy integral and is a direct consequence of Theorem~2.1 of \citet{van2011local}.

\begin{lemma}\label{lemma:local_rademacher_entropy}
Let \(\mathcal{F}\) be a star-shaped class of functions such that
\(\sup_{f\in\mathcal F}\|f\|_\infty \le M\). Then, for every \(\delta>0\),
\[
\mathcal{R}_n(\mathcal{F}, \delta)
\;\lesssim\;
\frac{1}{\sqrt n}\,\mathcal{J}(\delta,\mathcal{F})
\left(1+\frac{\mathcal{J}(\delta,\mathcal{F})}{\delta \sqrt n}\right),
\]
where the implicit constant depends only on \(M\).
\end{lemma}
\begin{proof}
This bound follows directly from the argument in the proof of
Theorem~2.1 of \citet{van2011local}; see in particular the step where
the local Rademacher complexity is controlled by the uniform entropy
integral for star-shaped classes.
\end{proof}

\section{Proof of Theorem \ref{theorem::Calerrorhist}}

\subsection{Technical lemmas}

\begin{lemma}
\label{lemma::entropynumbers}
Under our conditions,
\[
\mathcal{J}\left(\delta, \widehat{\mathcal{G}}\right)
\;\lesssim\;
\delta \sqrt{B \log(1/\delta)},
\]
where the implicit constant is independent of \(B\).
\end{lemma}
\begin{proof}

By Condition~\ref{cond::C2},
\[
\|\hat q_f-\hat q_g\|_{n,S,A}
\le
L\|f-g\|_{n,S'}
\]
almost surely.
For each $s$, by Jensen's inequality,
\[
|\pi h(s)|^2=\Big|\sum_a \pi(a\mid s)h(a,s)\Big|^2
\le \sum_a \pi(a\mid s)|h(a,s)|^2,
\]
so $\|\pi h\|_{n,S}\le \|h\|_{n,S,A}$. Hence, we also have
\[
\|\pi(\hat q_f-\hat q_g)\|_{n,S}
\le \|\hat q_f-\hat q_g\|_{n,S,A}
\le L \|f - g\|_{n, S'} .
\]
Hence, by the definition of \(\widehat Y_f\) and boundedness of
\(\widehat w_\pi\),
\[
\|\widehat Y_f - \widehat Y_g\|_{L^2(P_n)}
\lesssim
\|f - g\|_{n,S'}.
\]
Take $f, g \in \widehat{\mathcal{G}}$ with
$f = (f_1 - f_2)\big(\widehat Y_{f_2} - f_2\big)$
and
$g = (g_1 - g_2)\big(\widehat Y_{g_2} - g_2\big)$
for $f_1, f_2, g_1, g_2 \in \mathcal H_B(\hat v)$.
By Lipschitz continuity of multiplication and boundedness of nuisances, we have
\[
\|f - g\|_{L^2(P_n)}
\lesssim
\|f_1 - g_1\|_{n, S} + \|f_2 - g_2\|_{n, S}+ \|f_2 - g_2\|_{n, S'}.
\]
Taking the supremum over all discrete distributions $Q$, we obtain
\[
\|f - g\|_{L^2(P_n)}
\lesssim
\sup_Q \|f_1 - g_1\|_{L^2(Q)}
+ \sup_Q \|f_2 - g_2\|_{L^2(Q)} .
\]
Hence, by preservation of entropy integrals in \citet{van1996weak},
\[
\log N(\varepsilon, \widehat{\mathcal{G}}, L^2(P_n))
\lesssim
\sup_Q \log N(\varepsilon, \mathcal H_B(\hat v), L^2(Q)) .
\]
Taking the supremum over discrete distributions $Q$ on both sides yields the uniform covering number bound
\[
\sup_Q \log N(\varepsilon, \widehat{\mathcal{G}}, L^2(Q))
\lesssim
\sup_Q \log N(\varepsilon, \mathcal H_B(\hat v), L^2(Q)) .
\]

The class \(\mathcal H_B(\hat v)\) satisfies
\[
\sup_Q \log N\big(\varepsilon, \mathcal H_B(\hat v), L^2(Q)\big)
\;\le\;
\sup_{Q} \log N\big(\varepsilon, \widetilde{\mathcal{F}}_{B},
L^2(Q \circ \hat{v}^{-1})\big),
\]
where \(Q \circ \hat{v}^{-1}\) denotes the pushforward of \(Q\) under
\(\hat{v}\). Hence,
\[
\sup_Q \log N\big(\varepsilon, \mathcal H_B(\hat v), L^2(Q)\big)
\;\le\;
\sup_{Q} \log N\big(\varepsilon, \widetilde{\mathcal{F}}_{B}, L^2(Q)\big),
\]
where, by a slight abuse of notation, the supremum on the right-hand side is
taken over all discrete probability distributions \(Q\) on \(\mathbb{R}\).
The class \(\widetilde{\mathcal{F}}_{B}\) consists of all piecewise-constant
functions on \(\mathbb{R}\) taking at most \(B\) values. Therefore,
\(\widetilde{\mathcal{F}}_{B}\) has VC--subgraph dimension \(O(B)\), and
\citet{van1996weak} implies
\begin{align*}
    \sup_Q \log N(\varepsilon, \mathcal H_B(\hat v), L^2(Q))
&\lesssim
\sup_{Q} \log N\big(\varepsilon, \widetilde{\mathcal{F}}_{B}, L^2(Q)\big)\\
&\lesssim B \log(1/\varepsilon).
\end{align*}
 Consequently, $\mathcal{J}(\delta, \mathcal H_B(\hat v))
\lesssim \delta \sqrt{B \log(1/\delta)}$.
\end{proof}

\subsection{Proof of Theorem \ref{theorem::Calerrorhist}}
\begin{proof}[Proof of Theorem \ref{theorem::Calerrorhist}]
Denote $O_i := (S_i, A_i, R_i, S_i')$. The first-order optimality conditions for
$\theta_n^{(K)}$ imply that, for each bin $b \in [B]$,
\[
\frac{1}{n}\sum_{i=1}^n
\mathbf{1}_{I_b}(\hat{v}(S_i))
\big\{
\widehat Y_{V_{K-1}}(O_i)
- V_K(S_i)
\big\}
= 0.
\]
Hence, for any $f \in \mathcal G_B$, which is a linear combination of these
indicators,
\begin{equation}
\label{eqn::scores}
\frac{1}{n}\sum_{i=1}^n
f(\hat{v}(S_i))
\big\{
\widehat Y_{V_{K-1}}(O_i)
- V_K(S_i)
\big\}
= 0.
\end{equation}
Moreover, for any function $g:\mathbb{R}\to\mathbb{R}$, we may write
$g \circ V_K = f \circ \hat{v}$ for some $f$, and therefore,
for all such $g$,
\[
\frac{1}{n}\sum_{i=1}^n
g(V_K(S_i))
\big\{
\widehat Y_{V_{K-1}}(O_i)
- V_K(S_i)
\big\}
= 0.
\]
Noting that \(\Gamma_0(V_K) - V_K\) is of the form
\(g \circ V_K\) for some function \(g\), we have
\begin{align*}
0
&=
\frac{1}{n}\sum_{i=1}^n
\Big\{
\Gamma_0(V_K)(S_i) - V_K(S_i)
\Big\} \\
&\qquad\qquad\times
\Big\{
\widehat Y_{V_{K-1}}(O_i)
- V_K(S_i)
\Big\}.
\end{align*}
Adding and subtracting \(P_0\) yields
\begin{equation}
\label{eqn::basicineq}
\begin{aligned}
&P_0
\Big\{
\Gamma_0(V_K) - V_K
\Big\}
\Big\{
\widehat Y_{V_K}
- V_K
\Big\}\\
& = (P_0 - P_n)
\Big\{
\Gamma_0(V_K) - V_K
\Big\}
\Big\{
\widehat Y_{V_{K-1}}
- V_K
\Big\}\\
&\quad  + P_0
\Big\{
\Gamma_0(V_K) - V_K
\Big\}
\Big\{
\widehat Y_{V_K}
- \widehat Y_{V_{K-1}}
\Big\}.
\end{aligned}
\end{equation}

We rewrite the left-hand side of \eqref{eqn::basicineq}.
Adding and subtracting and applying the law of total expectation, the calibration error decomposes as
\begin{equation}
\label{eqn::calerrorbound}
\begin{aligned}
&\| \Gamma_0(V_K) - V_K\|^2\\
&=
P_0
\Big\{
\Gamma_0(V_K) - V_K
\Big\}
\Big\{
\widehat Y_{V_K}
- V_K
\Big\} \\
&\quad+
P_0
\Big\{
\Gamma_0(V_K) - V_K
\Big\}
\Big\{
\Gamma_0(V_K)
- \widehat Y_{V_K}
\Big\}.
\end{aligned}
\end{equation}

The first term on the right-hand side can be decomposed as in \eqref{eqn::basicineq}.
By the law of total expectation, the proof of Theorem~\ref{theorem::DRpseudo},
and the Cauchy--Schwarz inequality, the second term satisfies
\begin{align*}
&\left|
P_0
\Big\{
\Gamma_0(V_K) - V_K
\Big\}
\Big\{
\widehat Y_{V_K}
- \Gamma_0(V_K)
\Big\}
\right|
\\[0.4em]
&=
\left|
P_0
\Big\{
\Gamma_0(V_K) - V_K
\Big\}
\Big\{
\widehat Y_{V_K}
- \mathcal{T}_{\pi}(V_K)
\Big\}
\right|
\\[0.4em]
&\le\;
\|\Gamma_0(V_K) - V_K\|\;
\big\|
b_0\{(\widehat w_{\pi} - w_{\pi})\,
(\widehat q_{V_K} - q_{V_K})\}
\big\|_{2,\rho}.
\end{align*}

Hence, \eqref{eqn::basicineq} and \eqref{eqn::calerrorbound} together imply that
\begin{equation}
\label{eqn::basicineq2}
\begin{aligned}
&\|\Gamma_0(V_K) - V_K\|^2\\
&\le\;
(P_0 - P_n)
\Big\{
\Gamma_0(V_K) - V_K
\Big\}
\Big\{
\widehat Y_{V_{K-1}}
- V_K
\Big\}
\\[0.4em]
&\quad+\;
\|\Gamma_0(V_K) - V_K\|\;
\big\|
\widehat{\mathcal T}_\pi V_K
- \widehat{\mathcal T}_\pi V_{K-1}
\big\|
\\[0.4em]
&\quad+\;
\|\Gamma_0(V_K) - V_K\|\;
\big\|
b_0\{(\widehat w_{\pi} - w_{\pi})\,
(\widehat q_{V_K} - q_{V_K})\}
\big\|_{2,\rho}.
\end{aligned}
\end{equation}
where $\widehat{\mathcal T}_\pi v(s)
:= \mathbb{E}[\widehat Y_v(O)\mid S=s]$.
Here, the second term on the right-hand side follows from \eqref{eqn::basicineq}, noting that
\begin{align*} &\left| P_0 \Big\{ \Gamma_0(V_K) - V_K \Big\} \Big\{ \widehat Y_{V_K} - \widehat Y_{V_{K-1}} \Big\} \right| \\[0.4em] &= \left| P_0 \Big\{ \Gamma_0(V_K) - V_K \Big\} \,\mathbb{E}\left[ \widehat Y_{V_K} - \widehat Y_{V_{K-1}} \mid S \right] \right| \\[0.4em] &\le \|\Gamma_0(V_K) - V_K\|\; \Big\| \mathbb{E}\left[ \widehat Y_{V_K} - \widehat Y_{V_{K-1}} \mid S \right] \Big\| \\[0.4em] &= \|\Gamma_0(V_K) - V_K\|\; \big\| \widehat{\mathcal T}_\pi V_K - \widehat{\mathcal T}_\pi V_{K-1} \big\|_{2,\rho}. \end{align*}

We now turn to bounding the empirical process term on the right-hand side of \eqref{eqn::basicineq2}. Observe that
\[
\big(\Gamma_0(V_K) - V_K\big)
\big(\widehat Y_{V_{K-1}} - V_K\big)
\]
lies in a uniformly bounded subset of the class
\[
\widehat{\mathcal{G}} := \big\{(f_1 - f_2)\big(\widehat Y_{f_2} - f_2\big) :
f_1, f_2 \in \mathcal H_B(\hat v) \big\},
\]
By Lemma~\ref{lemma::entropynumbers}, it holds that $\mathcal{J}\left(\delta, \widehat{\mathcal{G}}\right)
\;\lesssim\;
\delta \sqrt{B \log(1/\delta)}.$
By assumption, \(\widehat{\mathcal{G}}\) is a fixed, nonrandom function class
conditional on the training data, which is independent of the calibration
sample \(\mathcal{C}_n\). Hence, applying Lemma~\ref{lemma:local_rademacher_entropy}
conditional on the training data, we obtain
\begin{align*}
\mathcal{R}_n(\widehat{\mathcal{G}}, \delta)
&\;\lesssim\;
\frac{1}{\sqrt n}\,
\mathcal{J}(\delta,\widehat{\mathcal{G}})
\left(1+\frac{\mathcal{J}(\delta,\widehat{\mathcal{G}})}{\delta \sqrt n}\right) \\[0.4em]
&\;\lesssim\;
\frac{1}{\sqrt n}\,
\delta \sqrt{B \log(1/\delta)}
\left(
1 +
\frac{\delta \sqrt{B \log(1/\delta)}}{\delta \sqrt n}
\right) \\[0.4em]
&\;\lesssim\;
\frac{\delta \sqrt{B \log(1/\delta)}}{\sqrt n}
\;+\;
\frac{B \log(1/\delta)}{n}.
\end{align*}
The critical radius
\[
\delta_n := \sqrt{\frac{B}{n}\,\log\Big(\frac{n}{B}\Big)},
\]
satisfies
\[
\delta_n
= \inf\{\delta > 0 : \mathcal{R}_n(\widehat{\mathcal{G}}, \delta)
\;\lesssim\; \delta^2\}.
\]

Applying Lemma~\ref{lemma:loc_max_ineq} conditional on the training data
with \(\mathcal{F} := \widehat{\mathcal{G}}\), we conclude that the following holds with probability at least
\(1 - e^{-u^2}\) for every \(f \in \widehat{\mathcal{G}}\):
\[
\frac{1}{n} \sum_{i=1}^n
\big(f(O_i) - \mathbb{E}[f(O_i)]\big)
\;\lesssim\;
\delta_n^2
\;+\;
\delta_n\|f\|_{2,P_0}
\;+\;
\frac{u\|f\|_{2,P_0}}{\sqrt{n}}
\;+\;
\frac{u^2}{n}.
\]
Choosing \(u=\sqrt{\log(1/\eta)}\) gives \(1-e^{-u^2}=1-\eta\). Hence, with probability at least \(1-\eta\),
\begin{align*}
&\frac{1}{n} \sum_{i=1}^n
\big(f(O_i) - \mathbb{E}[f(O_i)]\big)\\
&\lesssim
\delta_n^2
+
\delta_n\|f\|_{2,P_0}  +
\frac{\sqrt{\log(1/\eta)}\,\|f\|_{2,P_0}}{\sqrt{n}}
 +
\frac{\log(1/\eta)}{n}.
\end{align*}

By boundedness of our nuisances,
\[
\big\|
\{\Gamma_0(V_K) - V_K\}
\{\widehat Y_{V_{K-1}} - V_K\}
\big\|
\;\lesssim\;
\|\Gamma_0(V_K) - V_K\|_{2,\rho}.
\]
Hence, with probability at least \(1 - \eta\),
\[
\begin{aligned}
(P_n - P_0)
&\Big\{
\Gamma_0(V_K) - V_K
\Big\}
\Big\{
\widehat Y_{V_{K-1}}
- V_K
\Big\} \\
&\lesssim\;
\delta_n^2
\;+\;
\delta_n \|\Gamma_0(V_K) - V_K\| \\[0.35em]
&\quad+\;
\frac{\sqrt{\log(1/\eta)}\,\|\Gamma_0(V_K) - V_K\|}{\sqrt{n}}
\;+\;
\frac{\log(1/\eta)}{n},
\end{aligned}
\]
where the implicit constants do not depend on \(B\).

Combining the above with \eqref{eqn::basicineq2}, we find with probability at least \(1 - \eta\),
\begin{equation}
\label{eqn::basicineq3}
\begin{aligned}
&\|\Gamma_0(V_K) - V_K\|^2\\
&\lesssim\;
\delta_n^2
\;+\;
\delta_n \|\Gamma_0(V_K) - V_K\|\\
&\quad+\;
\frac{\sqrt{\log(1/\eta)}}{\sqrt{n}}\,
\|\Gamma_0(V_K) - V_K\|
\;+\;
\frac{\log(1/\eta)}{n}
\\[0.4em]
&\quad+\;
\|\Gamma_0(V_K) - V_K\|\;
\big\|
\widehat{\mathcal T}_\pi V_K
- \widehat{\mathcal T}_\pi V_{K-1}
\big\|
\\[0.4em]
&\quad+\;
\|\Gamma_0(V_K) - V_K\|\;
\big\|
b_0\{(\widehat w_{\pi} - w_{\pi})\,
(\widehat q_{V_K} - q_{V_K})\}
\big\|_{2,\rho}.
\end{aligned}
\end{equation}
The inequality in \eqref{eqn::basicineq3} implies that,
with probability at least $1-\eta$, the calibration error satisfies
\begin{align*}
\|\Gamma_0(V_K) - V_K\|
\;\lesssim\;&\;
\delta_n
\;+\;
\sqrt{\frac{\log(1/\eta)}{n}}
\\[0.4em]
&\;+\;
\big\|
\widehat{\mathcal T}_\pi V_K
- \widehat{\mathcal T}_\pi V_{K-1}
\big\|
\\[0.4em]
&\;+\;
\big\|
b_0\{(\widehat w_{\pi} - w_{\pi})\,
(\widehat q_{V_K} - q_{V_K})\}
\big\|_{2,\rho}.
\end{align*}
Recall that $\delta_n := \sqrt{\frac{B}{n}\,\log\Big(\frac{n}{B}\Big)}$. Then,  with probability at least $1-\eta$,
\begin{align*}
\|\Gamma_0(V_K) - V_K\|
\;\lesssim\;&\;
\sqrt{\frac{B}{n}\,\log\Big(\frac{n}{B}\Big)}
\;+\;
\sqrt{\frac{\log(1/\eta)}{n}}
\\[0.4em]
&\;+\;
\big\|
\widehat{\mathcal T}_\pi V_K
- \widehat{\mathcal T}_\pi V_{K-1}
\big\|
\\[0.4em]
&\;+\;
\big\|
b_0\{(\widehat w_{\pi} - w_{\pi})\,
(\widehat q_{V_K} - q_{V_K})\}
\big\|_{2,\rho}.
\end{align*}
\end{proof}

\section{Proof of Theorem \ref{theorem::regretHist}}

\subsection{Additional notation}
Define the bin-index map $\hat{b}(s)$ by setting $\hat{b}(s) = b$ whenever
$\hat{v}(s) \in I_b$.  Let
$\Pi_B$ denote the binning projection
\[
(\Pi_B g)(s)
:= E\!\left[g(S)\,\middle|\,\hat v(S)\in I_{\hat b(s)}\right].
\] Denote
\[
\mathcal B_{\pi,B}(g)(s)
:= (\Pi_B \mathcal{T}_{\pi}g)(s).
\]
Define $V_B^\star$ as the fixed point satisfying
\[
\mathcal B_{\pi,B}(V_B^\star) = V_B^\star.
\]
Such a fixed point exists by Banach's fixed point theorem because
$\mathcal B_{\pi,B}(g) \in \mathcal H_B(\hat v)$ whenever
$g \in \mathcal H_B(\hat v)$, and the operator is a $\gamma$-contraction in the
sup-norm.

Recall the projected
transition operator
\[
(P_{\pi,B} h)(s)
=
\mathbb{E}_{\pi}\left[h(S') \,\middle|\, \hat{v}(S) \in I_{\hat{b}(s)} \right].
\]
It holds that
\[
P_{\pi,B}
\;:=\;
\Pi_B P_\pi,
\]
and
\[
\Pi_B\mathcal{T}_\pi(f)
-
\Pi_B\mathcal{T}_\pi(g)
=
\gamma\, P_{\pi,B}(f-g).
\]

\subsection{Supporting lemmas}
We begin with the deterministic ingredients for the histogram value bound.

Throughout, histogram bins mean the retained cells of the fitted partition:
empty empirical cells are merged with a neighboring cell or discarded, and
\(B\) denotes the number of retained cells. Population projections are defined
on the corresponding non-null cells; \(\rho\)-null cells are identified in
\(L^2(\rho)\) and omitted from the index set.

\begin{lemma}[Sup-norm contraction of the aggregated Bellman operator]
\label{lemma::agg-sup-contraction}
For any bounded \(f,g:\mathcal S\to\mathbb R\),
\[
\|\mathcal B_{\pi,B}f-\mathcal B_{\pi,B}g\|_\infty
\le
\gamma\|f-g\|_\infty .
\]
\end{lemma}

\begin{proof}
Since the rewards cancel on differences,
\[
\mathcal B_{\pi,B}f-\mathcal B_{\pi,B}g
=
\gamma P_{\pi,B}(f-g).
\]
The operator \(P_{\pi,B}\) is a Markov averaging
operator. Hence, for any bounded \(h\),
\[
|P_{\pi,B}h(s)|
\le
E_\pi[|h(S')|\mid \hat v(S)\in I_{\hat b(s)}]
\le
\|h\|_\infty .
\]
Taking \(h=f-g\) proves the claim.
\end{proof}

\begin{lemma}[Approximation error under aggregated concentrability]
\label{lemma::approxerror}
Assume \ref{cond::Ccoverage}. Then
\[
\|V_B^\star - v_0\|_{2,\rho}
\le
\frac{\sqrt{\mathcal C_{B,\gamma}}}{1-\gamma}\,
\|(I-\Pi_B)v_0\|_{2,\rho}.
\]
\end{lemma}

\begin{proof}
Because \(v_0=\mathcal T_\pi v_0\) and
\(V_B^\star=\Pi_B\mathcal T_\pi V_B^\star\),
\[
V_B^\star-v_0
=
\gamma P_{\pi,B}(V_B^\star-v_0)
+
(\Pi_B-I)v_0 .
\]
Let \(a:=(\Pi_B-I)v_0\). Iterating the display and using boundedness gives the
resolvent expansion
\[
V_B^\star-v_0
=
\sum_{t=0}^\infty \gamma^t P_{\pi,B}^t a .
\]
For any \(t\ge0\), set
\[
c_t
:=
\left\|
\frac{d(\rho P_{\pi,B}^t)}{d\rho}
\right\|_\infty .
\]
Jensen's inequality for the Markov operator \(P_{\pi,B}^t\) gives
\[
\|P_{\pi,B}^t a\|_{2,\rho}^2
\le
\int P_{\pi,B}^t(a^2)\,d\rho
=
\int a^2\,d(\rho P_{\pi,B}^t)
\le
c_t\|a\|_{2,\rho}^2 .
\]
Therefore,
\[
\|V_B^\star-v_0\|_{2,\rho}
\le
\sum_{t=0}^\infty \gamma^t\sqrt{c_t}\,\|a\|_{2,\rho}
=
\frac{\sqrt{\mathcal C_{B,\gamma}}}{1-\gamma}
\|(I-\Pi_B)v_0\|_{2,\rho}.
\]
\end{proof}

\begin{lemma}[Inexact iterations under aggregated concentrability]
\label{lemma::inexacterrors}
Assume \ref{cond::Ccoverage}. Let \(\{\eta_k\}_{k\ge1}\) be any sequence such
that
\[
\|V_k - \mathcal B_{\pi,B}(V_{k-1})\|_{2,\rho}
\le
\eta_k
\qquad\text{for all } k.
\]
For \(m\ge0\), set
\[
c_m
:=
\left\|
\frac{d(\rho P_{\pi,B}^m)}{d\rho}
\right\|_\infty .
\]
Then, for any \(K\ge1\),
\[
\|V_K - V_B^\star\|_{2,\rho}
\le
\gamma^K\|\hat v - V_B^\star\|_\infty
+
\sum_{j=1}^K
\gamma^{K-j}\sqrt{c_{K-j}}\,\eta_j .
\]
Consequently,
\[
\|V_K - V_B^\star\|_{2,\rho}
\le
\gamma^K\|\hat v - V_B^\star\|_\infty
+
\frac{\sqrt{\mathcal C_{B,\gamma}}}{1-\gamma}
\max_{1\le j\le K}\eta_j .
\]
\end{lemma}

\begin{proof}
Let
\[
e_k:=V_k-\mathcal B_{\pi,B}(V_{k-1}).
\]
Since \(V_B^\star=\mathcal B_{\pi,B}V_B^\star\),
\[
V_k-V_B^\star
=
\gamma P_{\pi,B}(V_{k-1}-V_B^\star)+e_k.
\]

Iterating this recursion yields
\[
V_K-V_B^\star
=
\gamma^K P_{\pi,B}^K(\hat v-V_B^\star)
+
\sum_{j=1}^K\gamma^{K-j}P_{\pi,B}^{K-j}e_j .
\]
The first term is bounded by
\(\gamma^K\|\hat v-V_B^\star\|_\infty\) because \(P_{\pi,B}\) is a Markov
operator. For the remaining terms, the same Jensen and change-of-measure
argument used in Lemma~\ref{lemma::approxerror} gives
\[
\|P_{\pi,B}^{K-j}e_j\|_{2,\rho}
\le
\sqrt{c_{K-j}}\,\|e_j\|_{2,\rho}
\le
\sqrt{c_{K-j}}\,\eta_j .
\]
The displayed bounds follow by the triangle inequality and the definition of
\(\mathcal C_{B,\gamma}\).
\end{proof}

\begin{lemma}[Error bound for inexact updates]
\label{lemma::errorperiter}
There exists \(C<\infty\) such that, with probability at least \(1-\delta\),
uniformly over \(k=1,\ldots,K\),
 \begin{align*}
&\|V_k -
  \mathcal B_{\pi,B}(V_{k-1})\|_{2,\rho}\\
 & \qquad \leq
C\left\{
\sqrt{\frac{B}{n}\,\log\Big(\frac{n}{B}\Big)}
\;+\;
\sqrt{\frac{\log(K/\delta)}{n}}
+
\big\|
b_0\{(\widehat w_{\pi} - w_{\pi})\,
(\widehat q_{V_{k-1}} - q_{V_{k-1}})\}
\big\|_{2,\rho}
\right\}
\end{align*}
 \end{lemma}
\begin{proof}
Equation~\ref{eqn::scores} in the proof of Theorem~\ref{theorem::Calerrorhist} shows that, for any transformation $f \in \mathcal G_B$ and any $k$,
\begin{equation*}
\frac{1}{n}\sum_{i=1}^n
f(\hat{v}(S_i))
\big\{
\widehat Y_{V_{k-1}}(O_i)
- V_k(S_i)
\big\}
= 0.
\end{equation*}
Applying the above with
$f \circ \hat{v} = V_B^{\star(k)} - V_k$,
where
\[
V_B^{\star(k)} := \mathcal B_{\pi,B}(V_{k-1}),
\]
we obtain
\begin{equation*}
\frac{1}{n}\sum_{i=1}^n
\big(V_B^{\star(k)} - V_k\big)(S_i)
\Big\{
\widehat Y_{V_{k-1}}(O_i)
-
V_k(S_i)
\Big\}
= 0.
\end{equation*}
Adding and subtracting $P_0$, we find
\begin{equation}
\begin{aligned}
&P_0\left[
  (V_B^{\star(k)} - V_k)
  \Big\{
    \widehat Y_{V_{k-1}}
    - V_k
  \Big\}
\right]\\
&=
(P_0 - P_n)\left[
  (V_B^{\star(k)} - V_k)
  \Big\{
    \widehat Y_{V_{k-1}}
    - V_k
  \Big\}
\right].
\end{aligned}
\label{eq:pn_p0_identity}
\end{equation}

First, we study the left-hand side of
\eqref{eq:pn_p0_identity}. We have
\begin{align*}
    & P_0\left[
  (V_B^{\star(k)} - V_k)
  \Big\{
    \widehat Y_{V_{k-1}}
    - V_k
  \Big\}
\right]\\
& = \|V_B^{\star(k)} - V_k\|_{2,\rho}^2\\
& \quad + P_0\left[
  (V_B^{\star(k)} - V_k)
  \Big\{
    \widehat Y_{V_{k-1}}
    - V_B^{\star(k)}
  \Big\}
\right].
\end{align*}
Since \(V_B^{\star(k)}=\Pi_B\mathcal T_\pi V_{k-1}\) and
\(V_B^{\star(k)}-V_k\in\mathcal H_B(\hat v)\), the projection identity
gives
\[
P_0\left[
  (V_B^{\star(k)}-V_k)
  \{V_B^{\star(k)}-\mathcal T_\pi V_{k-1}\}
\right]=0.
\]
Thus
\begin{align*}
& P_0\left[
  (V_B^{\star(k)} - V_k)
  \Big\{
    \widehat Y_{V_{k-1}}
    - V_B^{\star(k)}
  \Big\}
\right] \\[0.4em]
&=
P_0\left[
  (V_B^{\star(k)} - V_k)
  \Big\{
    \widehat Y_{V_{k-1}}
    - \mathcal T_{\pi}V_{k-1}
  \Big\}
\right].
\end{align*}
By Theorem~\ref{theorem::DRpseudo} and arguing as in the proof of
Theorem~\ref{theorem::Calerrorhist}, it follows that
\begin{align*}
&\left|
P_0\left[
  (V_B^{\star(k)} - V_k)
  \Big\{
    \widehat Y_{V_{k-1}}
    - V_B^{\star(k)}
  \Big\}
\right]
\right|\\
&\;\le\;
\|V_B^{\star(k)} - V_k\|_{2,\rho}\;
\big\|
b_0\{(\widehat w_{\pi} - w_{\pi})\,
(\widehat q_{V_{k-1}} - q_{V_{k-1}})\}
\big\|_{2,\rho}.
\end{align*}
Putting it all together,
\begin{align*}
&\|V_B^{\star(k)} - V_k\|_{2,\rho}^2
 \leq
P_0\left[
  (V_B^{\star(k)} - V_k)
  \Big\{
    \widehat Y_{V_{k-1}}
    - V_k
  \Big\}
\right] \\
&+
\|V_B^{\star(k)} - V_k\|_{2,\rho}\;
\big\|
b_0\{(\widehat w_{\pi} - w_{\pi})\,
(\widehat q_{V_{k-1}} - q_{V_{k-1}})\}
\big\|_{2,\rho}.
\end{align*}
Hence, by \eqref{eq:pn_p0_identity},
\begin{equation}
\begin{aligned}
&\|V_B^{\star(k)} - V_k\|_{2,\rho}^2 \\
& \quad \leq
(P_0 - P_n)\left[
  (V_B^{\star(k)} - V_k)
  \Big\{
    \widehat Y_{V_{k-1}}
    - V_k
  \Big\}
\right]\\
& \quad +
\|V_B^{\star(k)} - V_k\|_{2,\rho}\;
\big\|
b_0\{(\widehat w_{\pi} - w_{\pi})\,
(\widehat q_{V_{k-1}} - q_{V_{k-1}})\}
\big\|_{2,\rho}.
\end{aligned}
\label{eq:pn_p0_identity2}
\end{equation}

Next, we obtain a high-probability bound for the first term on the right-hand side of
\eqref{eq:pn_p0_identity2}.
Applying Lemmas~\ref{lemma:loc_max_ineq}, \ref{lemma:local_rademacher_entropy},
and \ref{lemma::entropynumbers}, and arguing as in the proof of
Theorem~\ref{theorem::Calerrorhist}, we obtain that, with probability at least
\(1 - \delta\),
\[
\begin{aligned}
\Big|(P_0 - P_n)
&\Big[
(V_B^{\star(k)} - V_k)
\big\{
\widehat Y_{V_{k-1}}
-
V_k
\big\}
\Big]\Big|
\\[0.4em]
&\lesssim\;
\delta_n^2
\;+\;
\delta_n\,\|V_B^{\star(k)} - V_k\|_{2,\rho}
\\[0.4em]
&\quad+\;
\frac{\sqrt{\log(1/\delta)}\,\|V_B^{\star(k)} - V_k\|_{2,\rho}}{\sqrt{n}}
\;+\;
\frac{\log(1/\delta)}{n},
\end{aligned}
\]
where \(\delta_n := \sqrt{(B/n)\,\log(n/B)}\), and the implicit constants do
not depend on \(B\).

Plugging this high probability bound into \eqref{eq:pn_p0_identity2},  we obtain that, with probability at least
\(1 - \delta\),
\begin{equation*}
\begin{aligned}
&\|V_B^{\star(k)} - V_k\|_{2,\rho}^2 \\
&\lesssim\;
\delta_n^2
\;+\;
\delta_n\,\|V_B^{\star(k)} - V_k\|_{2,\rho}
\\[0.4em]
&\quad+\;
\frac{\sqrt{\log(1/\delta)}\,\|V_B^{\star(k)} - V_k\|_{2,\rho}}{\sqrt{n}}
\;+\;
\frac{\log(1/\delta)}{n}\\
& \quad +
\|V_B^{\star(k)} - V_k\|_{2,\rho}\;
\big\|
b_0\{(\widehat w_{\pi} - w_{\pi})\,
(\widehat q_{V_{k-1}} - q_{V_{k-1}})\}
\big\|_{2,\rho}.
\end{aligned}
\end{equation*}
Recalling that $\delta_n := \sqrt{\frac{B}{n}\,\log\Big(\frac{n}{B}\Big)}$, the inequality above implies that,
with probability at least $1-\delta$,
 \begin{align*}
\|V_B^{\star(k)} - V_k\|_{2,\rho}
\;\lesssim\;&\;
\sqrt{\frac{B}{n}\,\log\Big(\frac{n}{B}\Big)}
\;+\;
\sqrt{\frac{\log(1/\delta)}{n}}
\\[0.4em]
&  +
\big\|
b_0\{(\widehat w_{\pi} - w_{\pi})\,
(\widehat q_{V_{k-1}} - q_{V_{k-1}})\}
\big\|_{2,\rho}.
\end{align*}
Taking the failure probability to be \(\delta/K\) and applying a union bound over
\(k=1,\ldots,K\) proves the result.
\end{proof}

\subsection{Proof of Theorem \ref{theorem::regretHist}}

\begin{proof}[Proof of Theorem \ref{theorem::regretHist}]
For \(C<\infty\) large enough, define
\begin{align*}
  \eta_k
&=
C \Bigg\{\sqrt{\frac{B}{n}\,\log\!\Big(\frac{n}{B}\Big)}
\;+\;
\sqrt{\frac{\log(K/\delta)}{n}}
\\
& \quad
\;+\;
\big\|
b_0\{(\widehat w_{\pi} - w_{\pi})\,
(\widehat q_{V_{k-1}} - q_{V_{k-1}})\}
\big\|_{2,\rho}
\Bigg\}
\end{align*}

By Lemma~\ref{lemma::errorperiter}, with probability at least \(1-\delta\),
uniformly over \(k=1,\ldots,K\),
\[
\|V_k-\mathcal B_{\pi,B}(V_{k-1})\|_{2,\rho}
\le
\eta_k .
\]
Lemma~\ref{lemma::inexacterrors} therefore gives
\begin{align*}
\|V_K - V_B^\star\|_{2,\rho}
\le\;&
\gamma^K\|\hat v - V_B^\star\|_\infty\\
&+
\frac{\sqrt{\mathcal C_{B,\gamma}}}{1-\gamma}
\max_{1\le j\le K}\eta_j .
\end{align*}
Combining this with Lemma~\ref{lemma::approxerror} and the triangle inequality
yields
\begin{align*}
\|V_K-v_0\|_{2,\rho}
\le\;&
\gamma^K\|\hat v - V_B^\star\|_\infty\\
&+
\frac{\sqrt{\mathcal C_{B,\gamma}}}{1-\gamma}
\left[
\|(I-\Pi_B)v_0\|_{2,\rho}
+
\max_{1\le j\le K}\eta_j
\right].
\end{align*}
Since \(\Pi_B\) is the \(L^2(\rho)\) projection onto
\(\mathcal H_B(\hat v)\),
\[
\|(I-\Pi_B)v_0\|_{2,\rho}
=
\min_{h\in\mathcal H_B(\hat v)}\|h-v_0\|_{2,\rho}.
\]
Substituting the definition of \(\eta_k\) and absorbing constants completes the
proof.
\end{proof}

\subsection{Bound on Successive Iteration Errors}
\label{appendix::convergence}
\begin{lemma}[Bound on Successive Iteration Errors]
Under the conditions of Theorem~\ref{theorem::regretHist},
\begin{align*}
&\|V_K - V_{K-1}\|_{2,\rho}\\
&\lesssim
\gamma^{K-1}\,
\|\hat v - V_B^\star\|_{\infty}
\\[0.8em]
&\quad+\;
\frac{\sqrt{\mathcal C_{B,\gamma}}}{1-\gamma}
C\Bigg(
\sqrt{\frac{B}{n}\,\log\!\Big(\frac{n}{B}\Big)}
\;+\;
\sqrt{\frac{\log(K/\delta)}{n}}
\\[0.8em]
&\qquad\qquad+
\max_{1\le j\le K}
\Big\|
b_0\{
(\widehat w_{\pi} - w_{\pi})\,
(\widehat q_{V_{j-1}} - q_{V_{j-1}})
\}
\big\|_{2,\rho}
\Bigg).
\end{align*}
\end{lemma}

\begin{proof}
Lemma~\ref{lemma::inexacterrors} and Lemma~\ref{lemma::errorperiter} give the
same bound for \(\|V_m-V_B^\star\|_{2,\rho}\), for \(m=K-1\) and \(m=K\), with
\(\gamma^m\|\hat v-V_B^\star\|_\infty\) as the finite-iteration term. By the
triangle inequality,
\begin{align*}
\|V_K - V_{K-1}\|_{2,\rho}
&\le\;
\|V_{K-1} - V_B^\star\|_{2,\rho}
\\[0.25em]
&\quad+\;
\|V_K - V_B^\star\|_{2,\rho}.
\end{align*}
Substituting the two bounds and absorbing constants proves the claim.
\end{proof}

\section{Proofs of Theorems \ref{theorem::Calerroriso} and \ref{theorem::regretiso}}

\begin{proof}[Proof of Theorem \ref{theorem::Calerroriso}]
The first-order optimality conditions for isotonic regression (equivalently, its
interpretation as a histogram estimator) imply that for any function
$f:\mathbb{R}\rightarrow\mathbb{R}$ that is a linear combination of the
indicator functions defining the isotonic partition,
\begin{equation*}
\frac{1}{n}\sum_{i=1}^n
f\!\big(V_K(S_i)\big)
\Big\{
\widehat Y_{V_{K-1}}(O_i)
-
V_K(S_i)
\Big\}
= 0.
\end{equation*}
See, for example, the proof of Lemma~C.1 in \citet{van2023causal} and
\citet{van2024stabilized}. Choosing $f$ appropriately yields
\begin{align*}
0
&=
\frac{1}{n}\sum_{i=1}^n
\Big\{
\Gamma_0(V_K)(S_i) - V_K(S_i)
\Big\} \\
&\qquad\qquad\times
\Big\{
\widehat Y_{V_{K-1}}(O_i)
- V_K(S_i)
\Big\},
\end{align*}
which is the same basic equality as \eqref{eqn::basicineq} in the proof of
Theorem~\ref{theorem::Calerrorhist}. The remainder of the argument proceeds
along the same lines with minor modifications.

Specifically, let $\mathcal{F}_{TV}$ denote the union of $\mathcal G_{\mathrm{iso}}$,
which is uniformly bounded by $2M$ under Condition~\ref{cond::C1}, with all
functions of bounded total variation bounded by the constant $C$ in
Condition~\ref{cond::C3}. By \citet{van1996weak}, this class satisfies the
uniform entropy integral bound $\mathcal{J}(\delta,\mathcal{F}_{TV})
\lesssim \sqrt{\delta}$.

Let
\[
m_K(t)
:=
E\!\left[(\mathcal T_\pi V_K)(S)\mid \hat v(S)=t,\mathcal C_n\right].
\]
Condition~\ref{cond::C3} states that \(m_K\) has bounded total variation. Since
\[
\Gamma_0(V_K)(S)=E\{m_K(\hat v(S))\mid V_K(S)\},
\]
and \(V_K\) is a nondecreasing function of \(\hat v\), Lemma 6 of
\citet{van2024stabilized} implies that the coarsened curve
\(\Gamma_0(V_K)\) also has bounded total variation and lies in
\(\mathcal{F}_{TV, \hat{v}} := \{f \circ \hat{v}: f \in
\mathcal{F}_{TV}\}\). Thus, it holds that
\[
\big(\Gamma_0(V_K) - V_K\big)
\big(\widehat Y_{V_{K-1}} - V_K\big)
\]
lies in a uniformly bounded subset of the class
\[
\widehat{\mathcal{G}}
:=
\Big\{
(f_1 - f_2)\big(\widehat Y_{f_2} - f_2\big)
:\;
f_1, f_2 \in \mathcal{F}_{TV, \hat{v}}
\Big\}.
\]
Arguing as in the proof of Theorem~\ref{theorem::Calerrorhist}, we have
\[
\mathcal{J}(\delta, \widehat{\mathcal{G}})
\lesssim
\mathcal{J}(\delta, \mathcal{F}_{TV})
\lesssim
\sqrt{\delta}.
\]
The proof now follows directly from Theorem~\ref{theorem::Calerrorhist} with this
new choice of $\widehat{\mathcal{G}}$ and its associated critical radius
$\delta_n = n^{-1/3}$ for monotone functions.
\end{proof}

 \begin{proof}[Proof of Theorem \ref{theorem::regretiso}]
The result follows directly from a union bound and the proofs of
Theorems~\ref{theorem::Calerrorhist} and \ref{theorem::regretHist}, with
$\mathcal G_B$ replaced by $\mathcal G_{B_n}$, the class of all
piecewise-constant functions with at most $B_n$ constant segments. In
particular, the entropy bound in Lemma~\ref{lemma::entropynumbers} continues to
apply to this class. Notably, the proofs of these theorems allow for
data-adaptive partitions and require only a deterministic upper bound on the
number of constant segments.
\end{proof}

\clearpage

\end{document}